\documentclass[11pt]{article}
\usepackage{amsmath,amssymb,amsfonts, amsthm}
\usepackage[margin=1in]{geometry}
\usepackage{subcaption}
\usepackage{multirow}
\usepackage{graphicx}
\usepackage{float}
\usepackage{natbib}
\usepackage{bbm}
\usepackage{amssymb}
\usepackage{natbib}
\usepackage{booktabs}
\usepackage{url}
\usepackage[colorlinks,linkcolor=red,citecolor=blue,urlcolor=magenta]{hyperref}
\usepackage{mathtools}  
\mathtoolsset{showonlyrefs=true}

\usepackage{algorithm}
\usepackage{algpseudocode}

\usepackage[mathscr]{eucal}
\usepackage{mathrsfs}

\newcommand{\be}{\begin{equation}}
\newcommand{\ee}{\end{equation}}
\newtheorem{lemma}{Lemma}
\newtheorem{theorem}{Theorem}

\newtheorem{definition}{Definition}

\newtheorem{proposition}{Proposition}

\newtheorem{assumption}{Assumption}

\newtheorem{remark}{Remark}

\usepackage{enumerate}
\usepackage{amsmath}

\newcommand{\E}{\mathbb{E}}

\newcommand{\Prob}{\mathbb{P}}
\usepackage{natbib}
\bibpunct[, ]{(}{)}{,}{a}{}{,}%
\newcommand{\bProof}{\begin{proof}{Proof.}}
\newcommand{\eProof}{\hfill\Halmos\\  \end{proof}}

\newcommand{\calB}{\mathcal{B}}
\newcommand{\calA}{\mathcal{A}}

\newcommand{\calF}{\mathcal{F}}

\newcommand{\calX}{\mathcal{X}}

\newcommand{\bpi}{\boldsymbol{\pi}}

\newcommand{\bea}{\begin{equation*}}
\newcommand{\eea}{\end{equation*}}

\DeclareMathOperator*{\argmax}{arg\,max}

\usepackage{xcolor}

\usepackage{comment}
\usepackage{pgfplots}
\pgfplotsset{compat=newest}
\usepackage{tikz}
\usetikzlibrary{arrows.meta}

\title{Reinforcement Learning for Continuous-Time Jump Markov Decision Processes with Applications to Network Dynamic Pricing}

\author{
	{Huiling Meng\footnote{Department of Systems Engineering and Engineering Management, The Chinese University of Hong Kong, Hong Kong, China.
    Email:hlmeng@link.cuhk.edu.hk}}
	\and
 {Ningyuan Chen\footnote{Rotman School of Management, University of Toronto, Toronto, Canada, Email: ningyuan.chen@utoronto.ca}}
 \and
    {Xuefeng Gao\footnote{Department of Systems Engineering and Engineering Management, The Chinese University of Hong Kong, Hong Kong, China.
    Email:xfgao@se.cuhk.edu.hk}}
}

\date{\today}

\begin{document}
\maketitle

\begin{abstract}
We study reinforcement learning (RL) in Continuous-Time Jump Markov Decision Processes (CTJMDPs) featuring general discrete state spaces (which need not possess a vector space structure) and continuous/discrete action spaces. The setup covers many well-known applications in operations such as multi-product dynamic pricing with capacitated resources \citep{gallego1997multiproduct}. To model the exploration-exploitation tradeoff, we formulate an entropy-regularized continuous-time control problem with stochastic policies. Recent continuous-time RL techniques such as $q$-learning for controlled diffusions in \citep{jia2023qlearning} focus on continuous state spaces $\mathbb{R}^d$ and rely heavily on semimartingale theory in $\mathbb{R}^d$ for their theoretical analysis. 
Consequently, their methods cannot be directly applied to CTJMDPs with general discrete state spaces, which may lack the algebraic addition and subtraction structures inherent to Euclidean spaces. To bridge this gap, we establish the theoretical foundations of $q$-learning for CTJMDPs and develop model-free $q$-learning algorithms. 
Compared to na\"{i}ve time discretization and approximating CTJMDPs using discrete-time MDPs, our approach has several conceptual and empirical benefits.
Numerical experiments in network dynamic pricing \citep{gallego1997multiproduct} show that our proposed RL algorithm reliably learns near-optimal policies and consistently outperforms standard benchmark methods, demonstrating superior solution quality and effective scalability to large-scale network instances.
\end{abstract}

\section{Introduction}

 Reinforcement Learning (RL) is a powerful approach that enables agents to learn optimal policies through interaction with their environment \citep{sutton2018reinforcement}. 
Despite the vast literature on RL, most existing work is grounded in discrete-time Markov Decision Processes (MDPs), which provide a mathematical framework for modeling sequential decision-making. However, many real-world physical systems, such as autonomous driving, high-frequency trading, and robot navigation, operate in continuous time, requiring real-time monitoring and decision-making. In particular, decisions are not necessarily made at fixed intervals, rendering discrete-time models inadequate.
For continuous-time decision problems, one can discretize time uniformly upfront and apply existing RL algorithms developed for discrete-time MDPs. However, this approach could be highly sensitive to the choice of discretization step and performs poorly with small time steps (see e.g. \cite{munos2006policy, park2021time, tallec2019making}). These limitations, coupled with the rich analytical tools in the continuous-time setting, have sparked a surge of interest in continuous-time RL in recent years \citep{dai2023learning, gao2026reinforcement, guo2022entropy, jia2023qlearning, wang2025reinforcement, wei2024unified, zhao2023policy}. 
Existing studies in this line of research predominantly focus on controlled systems with \textit{continuous} state spaces in $\mathbb{R}^d$, where system dynamics are governed by stochastic differential equations (SDEs).

In contrast, this paper focuses on continuous-time RL in \textit{discrete} state spaces. We consider the Continuous-Time Jump Markov Decision Process (CTJMDP) \citep{feinberg2004continuous} with a denumerable state space as our mathematical decision model, which is a natural continuous-time extension of discrete-time MDPs. The CTJMDPs we study have the following features: 1) the system is observed continuously and (deterministic) actions can be made at any point in time; 2) the state space is denumerable and need not possess a vector space structure such as $\mathbb{R}^d$; 3) the action space may be discrete or continuous; 4) 
the transition rates may be time-dependent; and 5) the reward structure includes both running reward rates and jump
rewards, both of which may be time-dependent and unbounded.
These features preclude the direct application of existing RL theories and necessitate novel methodological tools. Critically, CTJMDPs differ from Semi-Markov Decision Processes (SMDP) \citep[Chapter 11]{puterman2014markov}, in which decisions are restricted to state-transition epochs and cannot be updated continuously. Consequently, CTJMDPs provide an ideal framework for modeling pure-jump stochastic systems subject to real-time control, with broad applications in operations research such as queueing control \citep{bremaud1981point}, population management \citep{guo2015finite}, and dynamic pricing \citep{gallego1994optimal}. Despite their modeling flexibility and rich theoretical foundation \citep{miller1968finite, feinberg2004continuous, guo2015finite}, model-free RL algorithms for CTJMDPs, where underlying system dynamics are unknown, remain largely unexplored. We address this fundamental gap by developing interpretable and scalable RL algorithms for CTJMDPs.

\subsection{Contributions}
In this study, we focus on the finite-horizon episodic RL setting, where an agent repeatedly interacts with the environment across multiple episodes, aiming to learn an optimal policy of a CTJMDP. The main contributions of the paper are summarized below.

\begin{itemize}
\item First, we develop a comprehensive and principled framework for model-free RL in CTJMDPs, encompassing both theoretical foundations and algorithm design.

Our framework extends the continuous-time $q$-learning theory, originally developed by \cite{jia2023qlearning} for controlled SDEs as a continuous-time analogue of discrete-time $Q$-learning, to discrete-state jump systems. Unlike the $\mathbb{R}^d$-valued diffusions in \cite{jia2023qlearning} or the $\mathbb{R}^d$-valued jump-diffusions in \cite{gao2026reinforcement}, RL in CTJMDPs pose unique theoretical challenges: their discrete state spaces generally lack vector-space operations (e.g., addition, subtraction) and do not support differential calculus. As a result, the continuous-time RL analysis in \citep{jia2023qlearning, gao2026reinforcement}, which relies heavily on semimartingales in $\mathbb{R}^d$ and Itô's formula, fails in our setting. We overcome these challenges by establishing a Dynkin’s formula (Theorem~\ref{thm:Dynkin}) tailored to the observable state process of a CTJMDP under randomized Markov policies.
 This result enables a martingale characterization of the optimal value function and $q$-function using sample trajectories (Theorem~\ref{thm:optimal-J-q}). Crucially, our derivation (of Theorem~\ref{thm:optimal-J-q}) departs from the corresponding results in \citep{jia2023qlearning, gao2026reinforcement} by  accommodating the jump rewards inherent to CTJMDPs under milder conditions. This martingale characterization yields interpretable \(q\)-learning algorithms for CTJMDPs.
Moreover, unlike SDEs, sample paths of CTJMDPs are piecewise constant, with jumps occurring at discrete points in time. This unique feature allows for more accurate approximations of the integrals that arise in value functions and it is incorporated in our development of \(q\)-learning algorithms for CTJMDPs.

\item 
Second, we apply the framework to the classical network dynamic pricing problem of \citet{gallego1997multiproduct} in an airline network. The problem combines a finite horizon, a discrete state, a continuous price action, and unknown demand functions; shared flight legs make both the state and the pricing decision high-dimensional and coupled. In the small network, the learned policy achieves revenue within $2.39\%$ of the time-discretized DP benchmark and improves on the two heuristic policies in \citet{gallego1997multiproduct}. In the large network, which has approximately $5.88\times10^{13}$ states and an 18-dimensional continuous action space, the learned policy slightly outperforms both heuristics even though they use the known demand functions. These results show that the model-free $q$-learning algorithm can learn a near-optimal policy in the small network and scale to the large network, for which direct dynamic programming is infeasible. 
Appendix~\ref{sec:queueing_example} applies the same continuous-time RL framework to a finite-horizon dynamic server allocation problem, demonstrating its broad applicability to queueing control.

\end{itemize}

\subsection{Related Work}\label{sec.relatedwork}

Our work is most closely related to two research streams, including RL for continuous-time jump decision processes, and network revenue management. Below, we highlight the key distinctions between our work and existing studies.

\paragraph{RL for continuous-time jump decision models.}
Early foundational work introduced continuous time RL by adapting algorithms such as $Q$-learning to infinite-horizon SMDPs \citep{bradtke1995reinforcement, das1999solving}. A standard approach for infinite-horizon SMDPs is to apply uniformization, converting the system into an equivalent discrete-time MDP to leverage discrete-time RL techniques \citep[e.g.,][]{dai2022queueing}. However, uniformization fails for the finite-horizon CTJMDPs considered in this paper. Because optimal policies in finite-horizon settings are non-stationary and depend explicitly on the continuous time index $t$, the system cannot be mapped to an equivalent discrete-time MDP. From a theoretical standpoint, \citet{gao2022logarithmic} and \citet{gao2025square} recently established regret bounds for tabular continuous-time MDPs under infinite-horizon average-reward and finite-horizon episodic settings, respectively. Their decision models belong to the class of exponential SMDPs \citep{feinberg2004continuous}, which differ fundamentally from our CTJMDP formulation: their framework restricts decision epochs to jump times, whereas our framework permits continuous-time action adjustments between jump events. More recently, \citet{meng2024reinforcement} developed a policy-gradient algorithm specifically for continuous-time intensity control in choice-based network revenue management. In contrast, our work establishes a unified, model-free RL framework for general CTJMDPs; the intensity control setting in their work can be viewed as a special case within our broader framework.

As this work was being finalized, a concurrent pre-print by \cite{zhang2026continuous} appeared on arXiv. While they also explore reinforcement learning for controlled continuous-time Markov chains in discrete state spaces, specifically proposing continuous-time variants of proximal policy optimization to fine-tune discrete generative diffusion models, our work was developed independently and differs fundamentally in its theoretical results, algorithmic design, and target applications.

\paragraph{Network revenue management.}
The classical network dynamic pricing model of \citet{gallego1997multiproduct} considers multiple products that consume finite, shared resources. Customers arrive in continuous time, and the firm continuously adjusts a vector of prices based on the remaining capacities. The model is a natural CTJMDP: remaining capacities form a discrete state, sales generate state jumps, and prices form a continuous action space.
Related work studies the same continuous-time and continuous-price formulation \citep[e.g.,][]{atar2013asymptotically}
Since then, there are many variants of the problem being proposed and studied.
Notably, the action space can become discrete (e.g., an accept/reject decision in \citealt{talluri1998analysis} or an assortment decision in \citealt{liu2008choice}) and the horizon can become discrete, e.g., in \citet{maglaras2006dynamic,zhang2013assessing,jasin2014reoptimization}.
These changes make the problem more tractable and allow discrete-time reinforcement learning approaches to be applied, e.g., \cite{li2026deep}.
However, the discretization of the horizon comes with approximation errors, which are typically hard to analyze and may induce numerical issues when the grid is too fine.
In this study, we retain the continuous-time and continuous-price
formulation of \citet{gallego1997multiproduct},
and develop a principled RL approach to solve the network dynamic pricing problem with arbitrary unknown demand functions. 

Related, there is extensive literature on demand learning for network revenue management.
They study the setting in \cite{gallego1997multiproduct} but the demand function is unknown initially.
\cite{besbes2012blind} is one of the earliest studies and uses the continuous-time setup.
Many of the studies use the discrete-time and continuous-action setup, \citep[e.g.,][]{chen2023network,miao2025network}.
This body of literature primarily focuses on designing efficient online learning algorithms with regret guarantees to balance the learning–earning tradeoff under demand uncertainty. Unlike these studies, our focus is to develop computational model-free RL methods that can handle high dimensionality and function approximation in massive state spaces. 
Several such studies have appeared in the literature. 
For instance, \citet{rana2015dynamic} applies online tabular $Q$-learning to a multi-product setting where products interact through cross-price demand effects with separate capacities, rather than competing for shared network resources. Meanwhile, \citet{hausenblas2025improving} adapts the offline actor-critic method to learn network dynamic pricing policies from historical sales data. These studies typically formulate the problem as a discrete-time MDP with a discrete-time and discrete-action setup. Hence, they are fundamentally different from our work.

\textbf{Notations.}
For any given function $w: \calX \mapsto (0, \infty)$, we introduce some notation associated with $w$ as follows.
A function $\varphi: [0, T] \times \calX \mapsto \mathbb{R}$ is called $w$-bounded if the $w$-weighted norm of $\varphi$, $\Vert \varphi \Vert_{w} \triangleq \sup_{(t, x) \in [0, T] \times \calX} \frac{\vert \varphi(t, x) \vert}{w (x)}$, is finite. Let $B_{w}([0, T] \times \calX)$ be the set of all $w$-bounded Borel measurable functions on $[0, T] \times \calX$; then it is a Banach space when endowed with the norm $\Vert\cdot\Vert_w$. 
Moreover, a function $\varphi: [0, T] \times \calX \mapsto \mathbb{R}$ is called essentially $w$-bounded if there exists a Lebesgue null set $Z_0 \subset [0, T]$ such that $\Vert \varphi \Vert_{w}^{\mathrm{es}} \triangleq \sup_{(t, x) \in [0, T]\backslash Z_0 \times \calX} \frac{\vert \varphi(t, x) \vert}{w(x)} < \infty$.  
We denote $C_{w}^{1, 0}([0, T] \times \calX)$ the set of functions $\varphi \in B_{w}([0, T] \times \mathcal{X})$ that satisfy: (i) $\varphi(t, x)$ is absolutely continuous in $t \in [0, T]$ for each fixed $x \in \calX$; (ii) $\frac{\partial \varphi}{\partial t}(t, x)$ is universally measurable and essentially $w$-bounded on $[0, T] \times \mathcal{X}$.

\section{Classical Finite-Horizon CTJMDPs}
\label{sec:classical-CTJMDP}
We consider a finite-horizon CTJMDP defined by a tuple
\begin{align*}
    \{T,\, \calX,\, \calA,\, \lambda(y \mid t, x, a),\, r(t, x, a),\, \rho(t,x,a,y),\, h(x)\},
\end{align*}
which are interpreted as follows:
\begin{itemize}
\item [(i)] $T > 0$ denotes the finite planning horizon;
\item [(ii)] $\calX$ is a denumerable (or countable) state space, which may not possess a vector space structure; 
\item [(iii)] $\mathcal{A} \subset \mathbb{R}^n$ is an action space, which can be either discrete or continuous; 

\item [(iv)]
For each $(t,x,a)\in\mathbb{M} \triangleq [0, T] \times \calX \times \calA$ and $y \neq x$,
$\lambda(y\mid t,x,a) \geq 0$ defines the transition rate from state $x$ to state $y$, and $\lambda(x \mid t, x, a) \triangleq - \sum_{y\neq x} \lambda(y \mid t,x,a)$.
For each $y\in\calX$, $\lambda(y\mid t,x,a)$ is Borel measurable on $\mathbb{M}$.
Moreover, the transition rates are stable in the sense that
\begin{align*}
    \lambda^*(x) \triangleq \sup_{t \in [0, T],\, a \in \mathcal{A}} \lambda(t, x, a) < \infty,
    \qquad x\in\calX.
\end{align*}
where $\lambda(t, x, a)\triangleq - \lambda(x \mid t, x, a) \geq 0$ for all $(t,x,a) \in \mathbb{M}$.

\item [(v)] $r(t,x,a)$ is a Borel measurable function defined on $\mathbb{M}$, which represents the running reward rate at time $t$ when the system
is in state $x$ under action $a$;

\item [(vi)] $\rho(t,x,a,y)$ is a Borel measurable function defined on
$\mathbb{M}$, which represents the jump reward
received upon a jump from state $x$ to state $y$ at time $t$ under
action $a$;

\item [(vii)] $h(x)$ is a real-value function on $\calX$, which denotes
the terminal reward incurred at the end of the planning horizon $T$;

\end{itemize}

A deterministic admissible control $\boldsymbol{u}$ refers to a measurable function $\boldsymbol{u}: [0, T] \times \calX \mapsto \calA$, and we denote the set of all admissible controls by $\mathcal{U}$.
For any given initial time-state pair $(t, x) \in [0, T] \times \calX$, let $X^{\boldsymbol{u}} = \{X^{\boldsymbol{u}}_s: s \in [t, T]\}$ with $X_t^{\boldsymbol{u}} = x$ be the controlled state process under an admissible control $\boldsymbol{u} \in \mathcal{U}$.
Specifically, the state process $X^{\boldsymbol{u}}$ evolves as follows.
Suppose the system is in state $X_s^{\boldsymbol{u}}=y$ at time $s$. While the system remains in state $y$, rewards accrue
continuously at rate
$r(s,y,\boldsymbol{u}(s,y))$.
Let $\tau$ denote the remaining sojourn time in state $y$. Given
$X_s^{\boldsymbol{u}}=y$, its survival function is
\begin{align*}
    \mathbb{P}(\tau>v\mid X_s^{\boldsymbol{u}}=y)
    =
    e^{
        -\int_0^v \lambda(s+s',y,\boldsymbol{u}(s+s',y))\,\mathrm{d}s'
    },
    \qquad v\geq 0.
\end{align*}
At time $s+\tau$, the system jumps to a state $z\neq y$ with 
probability $\frac{\lambda(z\mid s+\tau,y,\boldsymbol{u}(s+\tau,y))}
     {\lambda(s+\tau,y,\boldsymbol{u}(s+\tau,y))}$, 
and receives the jump reward
$\rho(s+\tau,y,\boldsymbol{u}(s+\tau,y),z)$.

Given an initial time-state pair $(t, x) \in [0, T] \times \calX$ and $\boldsymbol{u} \in \mathcal{U}$, for any $y,z\in\mathcal X$ with $y\neq z$, let
$N_{yz}^{\boldsymbol{u}}(s)$, $s\in[t,T]$, be the counting process that records the number of jumps of $X^{\boldsymbol{u}}$ from $y$ to $z$
over $(t,s]$.
The value function is
given by
\begin{align}\label{eq:value1}
    V(t, x; \boldsymbol{u}) 
    = \E \bigg[ \int_t^T r\left(s, X_{s-}^{\boldsymbol{u}}, \boldsymbol{u}(s, X_{s-}^{\boldsymbol{u}})\right) \mathrm{d}s 
    + \sum_{y\in\mathcal X}\sum_{z\neq y}
    \int_{(t,T]}
    \rho\bigl(s,y,\boldsymbol{u}(s,y),z\bigr)
    \,\mathrm{d}N_{yz}^{\boldsymbol{u}}(s)
    +h(X_T^{\boldsymbol{u}}) \mid X_{t}^{\boldsymbol{u}} = x\bigg].
\end{align}
The objective of the control problem is to find an optimal admissible control $\boldsymbol{u}^* \in \mathcal{U}$ that maximizes the value function
\begin{align}
\label{eq:classical-optimal-value-function}
    V^*(t, x) = \sup_{\boldsymbol{u}\in \mathcal{U}} V(t, x; \boldsymbol{u}), \quad \forall\, (t, x) \in [0, T] \times \calX.
\end{align}

We make the following assumptions regarding the transition rates, the running reward rates, the jump rewards, and the terminal reward throughout this paper. 
\begin{assumption}
\label{asp:transtion-reward-rate}
Suppose that $\sup_{x \in \calX} \lambda^*(x) < \infty$. Moreover, there exists a function $w: \calX \mapsto [1, \infty)$ and constants $c > 0$, $b \geq 0$, and $M_1 > 0$, 
such that the following
conditions hold for all $(t,x,a)\in\mathbb{M}$:
\begin{enumerate}[(i)]
    \item $\sum_{y\in\calX}w(y)\lambda(y\mid t,x,a) \leq cw(x)+b$.

    \item $|r(t,x,a)|\leq M_1w(x)$ and $|h(x)|\leq M_1w(x).$

    \item $\sum_{y\neq x}
        |\rho(t,x,a,y)|\lambda(y\mid t,x,a)
        \leq M_1w(x).$
\end{enumerate}
\end{assumption}

Assumption~\ref{asp:transtion-reward-rate}
allows the reward rates $r(t,x,a)$ and the terminal reward $h(x)$ to potentially grow with $x$. 
Here, we focus on CTJMDPs with bounded transition rates (i.e. $\sup_{x \in \calX} \lambda^*(x) < \infty$) to simplify the theoretical analysis. Our results can be readily extended to CTJMDPs with unbounded transition rates by
imposing conditions such as Assumptions~3.1 and~3.2 of
\citet{guo2015finite}. 

Indeed, it is useful to relate our formulation to the classical finite-horizon CTJMDP
model of \citet{guo2015finite}, which incorporates a running reward
rate and a terminal reward but no jump rewards. Define the augmented
running reward rate
\begin{align}\label{eq:augmented-reward}
    R(t,x,a)
    \triangleq
    r(t,x,a)
    +
    \sum_{y\neq x}
    \rho(t,x,a,y)\lambda(y\mid t,x,a).
\end{align}
By Assumption~\ref{asp:transtion-reward-rate}, one can readily verify that the compensated counting processes
$\{
N_{yz}^{\boldsymbol{u}}(s)
-
\int_t^s
\mathbf{1}_{\{X_{r-}^{\boldsymbol{u}}=y\}}
\lambda\!\left(z\mid r,y,\boldsymbol{u}(r,y)\right)\,\mathrm{d}r: s \in [t, T]\}$ is a martingale for $y,\, z\in\mathcal{X}$ with $y\neq z$. It follows that 
the value function \eqref{eq:value1} can be equivalently written as
\begin{align*}
    V(t,x;\boldsymbol{u})
    =
    \E\bigg[
        \int_t^T
        R\bigl(
            s,X_{s-}^{\boldsymbol{u}},
            \boldsymbol{u}(s,X_{s-}^{\boldsymbol{u}})
        \bigr)\,\mathrm{d}s
        +h(X_T^{\boldsymbol{u}})
        \mid
        X_t^{\boldsymbol{u}}=x
    \bigg].
\end{align*}
where $|R(t,x,a)| \leq 2M_1w(x)$.

Under Assumption \ref{asp:transtion-reward-rate}, the optimal value function $V^*(t, x)$ defined in \eqref{eq:classical-optimal-value-function} satisfies the following Hamilton-Jacobi-Bellman (HJB) equation (see Proposition 4.1 in \cite{guo2015finite}): 
\begin{align*} 
    \frac{\partial V^*}{\partial t }(t, x)  
     + \sup_{a \in \calA} H(t, x, a, V^*(\cdot, \cdot)) &=0, \quad (t, x) \in [0, T) \times \calX, \\
    V^*(T, x) &=h(x), \quad x \in \calX,  
\end{align*}
where the associated Hamiltonian $H: \mathbb{M} \times C^{1, 0}_{w}([0, T] \times \calX) \mapsto \mathbb{R}$ is defined as
\begin{align}
 \label{eq:Hamiltonian}
 H(t, x, a, v(\cdot, \cdot))  = R(t, x, a) +   \sum_{y \in \calX } v(t, y) \lambda(y \mid t, x, a), \quad (t, x, a) \in \mathbb{M},\, v \in C^{1, 0}_{w}([0, T] \times \calX). 
\end{align}
Moreover,
the optimal value is actually attained by an admissible control $\boldsymbol{u}^* \in \mathcal{U}$, which is given as
\begin{align*}
\boldsymbol{u}^*(t,x) = \argmax_{a \in \calA} H(t,x, a, V^*(t, x)), \quad \text{ for } (t, x) \in [0, T] \times \calX.
\end{align*}

\section{Reinforcement Learning Formulation for CTJMDPs}\label{sec:exploratory-formulation}

In this section, we discuss the formulation of continuous-time RL for CTJMDPs and consider randomized policies (for explorations) as in standard discrete-time RL.
To proceed, we introduce the definition of randomized Markov policies as follows. Let the action space \(\mathcal{A}\) be equipped with its Borel \(\sigma \)-algebra \(\mathcal{B}(\mathcal{A})\).

\begin{definition}\label{def:randomized-Markov-policy}
A randomized Markov policy is a stochastic kernel
$\bpi(\mathrm da\mid t,x)$ on $\calA$ given $[0,T]\times\calX$; that is,
\begin{itemize}
    \item[(i)] for each $(t,x)\in[0,T]\times\calX$,
    $\bpi(\cdot\mid t,x)$ is a probability measure on
    $(\calA,\, \calB(\calA))$;

    \item[(ii)] for each $x\in \calX$ and $S\in\calB(\calA)$, the mapping
    $t\mapsto\bpi(S\mid t,x)$ is measurable on $[0, T]$.
\end{itemize}
Moreover, a randomized Markov policy is called deterministic if there exists a
measurable decision rule
$\boldsymbol{u}:[0,T]\times\calX\to\calA$ such that $\bpi(\mathrm{d}a\mid t,x) =\delta_{\boldsymbol{u}(t,x)}(\mathrm{d}a)$ for all $(t,x)\in[0,T]\times\calX$.
\end{definition}
Next, we investigate the dynamics of the CTJMDP process under a randomized Markov policy $\bpi$.
It should be noted that the CTJMDP model, serving as a continuous-time analogue of the classical discrete-time MDP, enables actions to be altered at any time.
However, the construction of the CTJMDP under a randomized Markov policy may encounter measurability issues. 
Suppose we want to select actions independently at each time $t \in [0, T]$ according to the non-degenerate distribution $\bpi(\cdot \mid t, x_{t-})$; then, the action process $a_t$ may fail to be measurable, one can refer to Example 1.2.5 in \cite{kallianpur2013stochastic} for a concrete illustration.
Since the reward rate and transition rate are defined for the randomized action process, 
the absence of measurability makes it impossible to properly define the associated integrals, which in turn hinders the construction of the corresponding CTJMDP.

To avoid these obstacles, we follow the definition of CTJMDP under randomized Markov policies as presented in Section 7 of \cite{feinberg2004continuous}.
Specifically, for any randomized Markov policy $\bpi$, we consider a Markov process $\tilde{X}^{\bpi} = \{\tilde{X}_s^{\bpi}: s \in [t, T]\}$ starting from $\tilde{X}_t^{\bpi} = x$, with transition rate
\begin{align}
\label{eq:exploratory-transition-rate}
    \lambda^{\bpi}(y \mid t, x) \triangleq \int_{\calA} \lambda(y \mid t, x, a) \boldsymbol{\pi}(\mathrm{d}a \mid t, x), \quad \text{ for } (t, x) \in [0, T] \times \calX, \, y \in \calX,
\end{align}
and $\lambda^{\bpi}(t, x) \triangleq - \lambda^{\bpi}(x \mid t, x)$ for $(t, x) \in [0, T] \times \calX$.
We refer to $\tilde{X}^{\bpi}$ as the exploratory state process, which will be used for theoretical analysis only as it is not observable in practice.
The reward rate function associated with $\tilde{X}^{\bpi}$ is defined as 
\begin{align*}
    R^{\bpi}(t, x) \triangleq \int_{\calA} R(t, x, a) \bpi(\mathrm{d}a \mid t,x), \quad \text{ for } (t, x) \in [0, T] \times \calX,
\end{align*}
where $R(t,x,a)$ is defined in \eqref{eq:augmented-reward}.
Suppose the exploratory state process is in
state $\tilde{X}_s^{\bpi}=y$ at time $s$. While it remains in state
$y$, running rewards accrue continuously at rate $R^{\bpi}(s,y)$.
Let $\tau$ denote the remaining sojourn time in state $y$. Its
conditional survival function is $\mathbb{P}(
\tau>v
\mid
\tilde{X}_s^{\bpi}=y
)
=
e^{
-\int_0^v \lambda^{\bpi}(s+s',y)\,\mathrm{d}s'}$ for $v\geq0$.
Note that the above dynamics does not involve the particular randomized actions, it depends only on the policy measure itself, and can be viewed as the average of the dynamics over the randomization of actions.
Then, the value function of policy $\bpi$ is defined in terms of the exploratory state process $\tilde{X}^{\bpi}$. Specifically, for each $(t, x) \in [0, T] \times \calX$,
\begin{align}
 \label{eq:value-function-V-given-pi}
    V(t, x; \bpi) \triangleq \mathbb{E} \bigg[ \int_t^T R^{\bpi}(s, \tilde{X}_{s-}^{\bpi}) \mathrm{d}s + h (\tilde{X}_T^{\bpi}) \mid \tilde{X}_t^{\bpi} = x \bigg].
\end{align}

However, merely allowing
randomized Markov policies  does not ensure exploration needed for RL, because under mild conditions the finite-horizon CTJMDP admits an
optimal deterministic policy even when optimization is over all
randomized Markov policies
\citep[Theorem~4.1]{guo2015finite}. 
Following
entropy-regularized formulations in discrete-time RL
\citep{Haarnoja2018soft} and continuous-time RL for diffusions \citep{jia2023qlearning}, we introduce an
entropy bonus to encourage exploration. 
To provide a unified definition of entropy for discrete and continuous action spaces in \(\mathbb R^n\), fix a nonzero $\sigma$-finite reference
measure $\mu$ on $(\mathcal A,\mathcal B(\mathcal A))$.
For example, consider a pricing problem with $n$ products. If the admissible
price vectors form a finite menu
$\mathcal A\subset\mathbb R_+^n$, then $\mu$ may be taken as the counting
measure. If each product can be priced at any nonnegative level, then
$\mathcal A=\mathbb R_+^n$ and $\mu$ may be taken as the
$n$-dimensional Lebesgue measure restricted to $\mathbb R_+^n$.

For a randomized Markov policy $\bpi$ that is absolutely
continuous with respect to $\mu$, written
$\bpi(\cdot\mid t,x)\ll\mu$, we use $\bpi(a\mid t,x)$ to denote its
Radon-Nikodym density:
\[
    \bpi(\mathrm{d}a\mid t,x)
    =
    \bpi(a\mid t,x)\mu(\mathrm{d}a).
\]
The entropy of $\bpi(\cdot\mid t,x)$ relative to $\mu$ is defined 
by
\[ \mathcal{H}_{\mu}\bigl(\bpi(\cdot\mid t,x)\bigr) \triangleq -\int_{\mathcal{A}} \log\bpi(a\mid t,x)\, \bpi(\mathrm{d}a\mid t,x).
\]

For $\gamma>0$, the entropy-regularized value function for a CTJMDP associated with a
randomized Markov policy $\bpi$ is defined as
\begin{align}\label{eq:value-function-J-given-pi}
J(t, x; \boldsymbol{\pi}) & \triangleq V(t, x; \bpi) - \gamma\, \mathbb{E} \bigg[\int_t^T  \int_{\calA}  \log \boldsymbol{\pi}(a \mid s, \tilde X_{s-}^{\boldsymbol{\pi}})] \boldsymbol{\pi}(\mathrm{d}a  \mid s, \tilde X_{s-}^{\boldsymbol{\pi}}) \mathrm{d}s \mid \tilde X_t^{\boldsymbol{\pi}} = x \bigg] \\
&= \mathbb{E} \bigg[\int_t^T  \int_{\calA}  [ R(s, \tilde X_{s-}^{\boldsymbol{\pi}}, a) - \gamma \log \boldsymbol{\pi}(a \mid s, \tilde X_{s-}^{\boldsymbol{\pi}})  ] \boldsymbol{\pi}(\mathrm{d}a  \mid s, \tilde X_{s-}^{\boldsymbol{\pi}}) \mathrm{d}s + h(\tilde X_T^{\boldsymbol{\pi}} ) \mid \tilde X_t^{\boldsymbol{\pi}} = x \bigg], \nonumber
\end{align}
Here, $\gamma$ is the temperature parameter controlling the exploration weight, typically chosen as a small tuned constant in RL.
As $\gamma \rightarrow 0$, we have $J(t, x; \boldsymbol{\pi}) \rightarrow  V(t, x; \bpi)$ given in \eqref{eq:value-function-V-given-pi}. 
The optimal (entropy-regularized) value function is defined by
\begin{align}\label{eq:optimal-value-function-J}
    J^*(t,x)
    = \sup_{\bpi\in\boldsymbol{\Pi}}J(t,x;\bpi),
    \qquad
    (t,x)\in[0,T]\times\mathcal{X},
\end{align}
where $\boldsymbol{\Pi}$ denotes the class of
admissible randomized Markov policies specified below.
\begin{definition}\label{def:admissible-policy}  
A randomized Markov policy $\boldsymbol{\pi}$ is called admissible if:
\begin{itemize}
\item [(i)] 
    For each $(t, x) \in [0, T] \times \calX$, $\bpi(\cdot\mid t,x)\sim\mu$; that is, \[\bpi(\cdot\mid t,x)\ll\mu\quad \text{ and } \quad
    \mu\ll\bpi(\cdot\mid t,x).\]
\item [(ii)] 
There exists a constant $M_2 > 0$, such that 
\begin{align*}
   \vert \mathcal{H}_{\mu}(\bpi(\cdot \mid t, x)) \vert \leq M_2 w(x), \qquad (t, x) \in [0, T] \times \calX,
\end{align*}
where the function $w(\cdot)$ is as specified in Assumption \ref{asp:transtion-reward-rate}.
\end{itemize}
\end{definition}
Condition (i) serves two purposes. The relation
$\bpi(\cdot\mid t,x)\ll\mu$ makes the entropy term well defined,
whereas $\mu\ll\bpi(\cdot\mid t,x)$ ensures action coverage: every
measurable action set with positive $\mu$-measure is assigned positive
probability. Condition (ii) controls the growth of the entropy term
and ensures the finiteness of the regularized value function.

The following proposition establishes the corresponding
Hamilton-Jacobi-Bellman characterization of the optimal value function $J^*$.
It further provides an explicit characterization of the optimal stochastic policy,
which takes a soft-max form determined by the Hamiltonian (see \eqref{eq:Hamiltonian}) associated with
$J^*$.
\begin{proposition} \label{prop:HJB}
Suppose Assumption~\ref{asp:transtion-reward-rate} holds with function $w$. Then there exists a unique solution $\varphi \in C^{1, 0}_{w}([0, T] \times \calX)$ satisfying the terminal
condition $\varphi (T, x) = h (x)$ for all $x \in \calX$, and for Lebesgue-a.e. $t\in[0,T)$ and $x \in \calX$, 
\begin{align}\label{eq:HJB}
    \frac{\partial \varphi}{\partial t}(t, x) +  \sup_{ \boldsymbol{\pi} \in \boldsymbol{\Pi}} \int_{\calA} \{ H(t,x, a, \varphi(\cdot, \cdot) ) - \gamma \log \boldsymbol{\pi}(a \mid t, x)\} \boldsymbol{\pi}(\mathrm{d}a \mid t, x) = 0.
\end{align}
Moreover, the unique solution $\varphi \in C^{1, 0}_{w}([0, T) \times \calX)$ is the optimal value function as defined in \eqref{eq:optimal-value-function-J}, i.e., $\varphi(t, x) = J^*(t, x)$ for all $(t, x) \in [0, T] \times \calX$. If the  policy $\bpi^*$ defined by
\begin{align}\label{eq:optimal-policy}
\boldsymbol{\pi}^*(a \mid t, x) = \frac{\exp \{\frac{1}{\gamma} H(t, x, a, J^*(\cdot, \cdot)) \}}{\int_{\calA} \exp \{\frac{1}{\gamma} H(t, x, a', J^*(\cdot, \cdot))\} \mu(\mathrm{d} a')}, \qquad  (t, x, a) \in \mathbb{M}.
\end{align}
is admissible, then it is the optimal policy for \eqref{eq:optimal-value-function-J}.
\end{proposition}

\section{$q$-Learning Theory and Algorithm for CTJMDPs} \label{sec:q-learning}
The Hamilton-Jacobi-Bellman characterization in
Proposition~\ref{prop:HJB} is model-based, as evaluating the Hamiltonian \eqref{eq:Hamiltonian}
requires knowledge of the reward and transition rates. This section
develops a continuous-time $q$-learning framework that directly learns the
optimal value function $J^*$ and the so-called optimal $q$-function from observed trajectories of CTJMDPs,
without estimating model parameters, and recovers the optimal policy
from the learned $q$-function.

\subsection{The Optimal $q$-Function}
The continuous-time $q$-learning framework was introduced by
\citet{jia2023qlearning} for entropy-regularized diffusion control
problems. They observe that the conventional $Q$-function collapses to the
value function as the duration of the initial action vanishes and introduce
the $q$-function (as opposed to $Q$-function) to capture the first-order, action-dependent
term. The continuous-time $q$-function can be interpreted as the advantage rate function, an analogue of the advantage function commonly used in discrete-time RL. We adapt the construction of $q$-function to finite-horizon CTJMDPs and focus
directly on its optimal counterpart.

Define the optimal $q$-function for the
entropy-regularized CTJMDP as
\begin{align*}
q^*(t,x,a)
\triangleq
\frac{\partial J^*}{\partial t}(t,x)
+
H\bigl(t,x,a,J^*(\cdot,\cdot)\bigr), \qquad \quad (t, x, a) \in \mathbb{M}.
\end{align*}
From Proposition~\ref{prop:HJB}, one can show that the optimal $q$-function
satisfies (see the proof of Proposition~8 in \citealp{jia2023qlearning} for an
analogous argument):
\begin{align*}
\int_{\calA}
\exp\{\frac{1}{\gamma}q^*(t,x,a)\}
\mu(\mathrm{d}a)
=1.
\end{align*}
Thus, $\exp\{\frac{1}{\gamma}q^*(t,x,\cdot)\}$ defines a probability
density with respect to $\mu$. By the optimal-policy characterization in
Proposition~\ref{prop:HJB}, this density is precisely
\begin{align*}
\bpi^*(a\mid t,x)
=
\exp\{\frac{1}{\gamma}q^*(t,x,a)\}.
\end{align*}
Therefore, learning the optimal $q$-function directly recovers the optimal policy $\bpi$.
This suggests using $q^*$,
together with $J^*$, 
as our learning targets, similar as learning the conventional policy-value pair in discrete-time actor–critic RL methods. Note that $(J^*, q^*)$ are continuous-time-based targets and they are independent of time discretizations.

\subsection{Martingale Characterization and Main Theoretical Results}
We now develop a martingale characterization that jointly identifies
$(J^*,q^*)$ from observed trajectories. This characterization provides the
theoretical foundation for our learning algorithm. 
To formulate this characterization, we first specify how randomized
policies are implemented to generate trajectories.
As discussed in Section~\ref{sec:exploratory-formulation}, although a randomized
Markov policy specifies an action distribution at every time-state pair, independently sampling an action at each instant need not yield a measurable action process and is not operationally implementable. 
We therefore implement an admissible randomized policy on a
decision grid: actions are sampled only at the grid points and held fixed
between consecutive points, while the state process continues to evolve in
continuous time.

Formally, given an initial time-state pair
$(t,x)\in[0,T]\times\calX$, consider a time grid on $[t, T]$:  
\[\mathbb{S} \triangleq \{t = t_0 < t_1 < \cdots < t_{K} = T\}.\]
Let $\{U_k\}_{k=0}^{K-1}$ be independent
uniform random variables on $[0,1]^n$ used for action randomization ($n$ is the dimension of the action space).
Given a randomized policy $\bpi,$ there always exists some measurable function $\Phi: [0, T] \times \calX \times [0, 1]^n \mapsto \calA$ such that $\Phi(t, x, U) \sim \bpi(\cdot \mid t, x)$ for all $(t, x) \in [0, T] \times \calX$ and $U \sim U([0, 1]^n)$. 
We next implement $\bpi$ on the time grid $\mathbb S$. Let
$X^{\bpi,\mathbb S}
=\{X_s^{\bpi,\mathbb S}: s\in [t, T]\}$ denote the resulting (observable) state process with
$X_t^{\bpi,\mathbb S}=x$, which we refer to as the grid sample state
process.
For each $k=0,\ldots,K-1$, let $A_k
\triangleq
\Phi\bigl(t_k,X_{t_k}^{\bpi,\mathbb S},U_k\bigr)$
denote the action sampled at time $t_k$. The action process $a^{\bpi,\mathbb{S}} = \{a_s^{\bpi,\mathbb{S}}: s \in [t, T]\}$ is defined as
\begin{align*}
    a_s^{\bpi,\mathbb{S}}
    \triangleq \sum_{k=0}^{K-1} A_k \boldsymbol{1}_{(t_k,t_{k+1}]}(s), \qquad s \in (t, T].
\end{align*}
Let $\{\calF_s^{\mathbb{S}}\}_{s\in [t, T]}$ denote the filtration generated jointly by the state process and the action-randomization variables, where
\[
\calF_s^{\mathbb{S}} \triangleq \sigma\{X_u^{\bpi,\mathbb S}:t\leq u\leq s\}\vee \sigma\{U_k:t_k\leq s\}, \qquad s\in[t,T].
\]
The (observable) state process $X^{\bpi,\mathbb{S}}$ and the action process
$a^{\bpi,\mathbb{S}}$ are adapted to this filtration, and
$a^{\bpi,\mathbb{S}}$ is predictable.
For later use, let $\{\tau_{\ell}\}_{\ell\geq1}$ denote the successive jump times of $X^{\bpi,\mathbb S}$ on $(t,T]$.
For any $y,\, z \in \calX$, let $N_{yz}^{\bpi,\mathbb S}(s)$, $s\in[t,T]$, denote the counting process that records the number of jumps of $X^{\bpi,\mathbb S}$ from $y$ to $z$
over $(t,T]$. In what follows, expectations involving $X^{\bpi,\mathbb{S}}$ or $a^{\bpi,\mathbb{S}}$ are taken over both the state-transition randomness and the action randomization. With a slight abuse of notation, we continue to denote such expectations by $\E$.

We are now ready to state the first main theoretical result of this paper.
The following Dynkin's formula for the grid sample state process $X^{\bpi,\mathbb S}$ is crucial for the subsequent martingale
characterization of $(J^*, q^*)$. This result is non-trivial: while the standard Dynkin's formula applies to Markov processes, the grid sample state process $X^{\bpi,\mathbb S}$ is non-Markovian on $[0,T]$. Its future evolution depends not only on the current state, but also on the action selected at the preceding grid point.

\begin{theorem}[Dynkin's formula] \label{thm:Dynkin}
Suppose Assumption \ref{asp:transtion-reward-rate} holds with function $w$. Then, for any initial time-state pair $(t, x) \in [0, T] \times \calX$ and time grid $\mathbb{S}$ on $[t, T]$, the grid sample state process $\{X_s^{\bpi,\mathbb{S}}: s \in [t, T]\}$ satisfies: for each $\varphi \in C^{1, 0}_{w}([0, T] \times \calX)$ and $s \in [t, T]$, 
\begin{align*}
     \E \bigg[ \int_s^T  \bigg(\frac{\partial \varphi}{\partial u}(u, X_{u-}^{\bpi,\mathbb{S}}) + \sum_{y \in \calX} \varphi(u, y) \lambda(y \mid u, X_{u-}^{\bpi,\mathbb{S}}, a_{u}^{\bpi,\mathbb{S}}) \bigg) \mathrm{d}{u} \mid \calF_s^{\mathbb{S}} \bigg] = \E [\varphi(T, X_T^{\bpi,\mathbb{S}}) \mid \calF_s^{\mathbb{S}}] - \varphi(s, X_s^{\bpi,\mathbb{S}}).
\end{align*}
\end{theorem}
The proof of Theorem~\ref{thm:Dynkin} differs from the semimartingale argument
used by \citet{gao2026reinforcement}, which generalizes \citet{jia2023qlearning} to RL for jump-diffusions. In their jump-diffusion setting,
the grid sample state process is governed by an SDE driven by Brownian motion and Poisson random measures, which directly provides its semimartingale representation and permits an application of the It\^{o}'s formula for their theoretical analysis. 
A CTJMDP, by contrast, is specified through its transition-rate kernel on a general denumerable state space, which need not be a subset of $\mathbb R^d$. Consequently, the grid sample state process of a CTJMDP may not possess a semimartingale structure.
Theorem~\ref{thm:Dynkin}
establishes the Dynkin's formula for $X^{\bpi,\mathbb S}$ directly from the
transition rates.

We next state the second main theoretical result of this paper. 
The following theorem shows that, subject to the terminal and normalization conditions, the martingale property along
trajectories of the grid sample state process uniquely identifies the pair of optimal value function and optimal $q$-function. This result provides the theoretical foundation for our
proposed RL algorithm.
\begin{theorem}\label{thm:optimal-J-q}
Suppose Assumption \ref{asp:transtion-reward-rate} holds with function $w$. Let $\widehat{J^*} \in C_{w}^{1, 0}([0, T] \times \calX)$ and $\widehat{q^*}: \mathbb{M} \mapsto \mathbb{R}$ be measurable, satisfying 
    \begin{align}\label{eq:optimal-J-q-constraints}
        \widehat{J^*}(T, x) = h(x), \quad \int_{\calA} \exp\{\frac{1}{\gamma}\widehat{q^*}(t, x, a)\} \mu(\mathrm{d}a) = 1, \quad \forall\, (t, x) \in [0, T] \times \mathcal{X}.
    \end{align}
    Then, 
    \begin{itemize}
        \item [(i)] If $\widehat{J^*}$ and $\widehat{q^*}$ are respectively the optimal value function and the optimal $q$-function, then given any $\bpi \in \boldsymbol{\Pi}$, for all initial time-state pair $(t, x) \in [0, T] \times \calX$ and any time grid $\mathbb{S}$ on $[t, T]$, the following process 

        \begin{align*}
            \widehat{J^*}(s, X_s^{\bpi,\mathbb{S}}) + \int_t^s \sum_{y\in\mathcal X}\sum_{z\neq y} \rho\bigl(u,y,a_u^{\bpi^\psi,\mathbb{S}},z\bigr)\,\mathrm{d}N_{yz}^{\bpi,\mathbb{S}}(u) + \int_t^s [ r(u, X_{u-}^{\bpi,\mathbb{S}}, a_u^{\bpi,\mathbb{S}}) - \widehat{q^*}(u, X_{u-}^{\bpi,\mathbb{S}}, a_u^{\bpi,\mathbb{S}}) ] \mathrm{d}u 
        \end{align*}
        is an $\{\calF_s^{\mathbb{S}}\}_{s \in[t, T]}$-martingale.

        \item [(ii)] If there exists one $\bpi \in \boldsymbol{\Pi}$ and a time grid $\mathbb{S} = \{0 = t_0 < t_1 < \cdots < t_{K} = T\}$ on $[0, T]$, such that for all initial state  $x \in \calX$, the process
        \begin{align}\label{eq:optimal-martingale}
            \widehat{J^*}(t, X_t^{\bpi,\mathbb{S}}) + \int_0^t \sum_{y\in\mathcal X}\sum_{z\neq y} \rho\bigl(s,y,a_s^{\bpi^\psi,\mathbb{S}},z\bigr)\,\mathrm{d}N_{yz}^{\bpi,\mathbb{S}}(s) + \int_0^t [ r(s, X_{s-}^{\bpi,\mathbb{S}}, a_s^{\bpi,\mathbb{S}}) - \widehat{q^*}(s, X_{s-}^{\bpi,\mathbb{S}}, a_s^{\bpi,\mathbb{S}}) ] \mathrm{d}s
        \end{align}
        is an $\{\calF_t^{\mathbb{S}}\}_{t \in[0, T]}$-martingale, then $\widehat{J^*}$ and $\widehat{q^*}$ are respectively the optimal value function and the optimal $q$-function.
    \end{itemize}
\end{theorem}
Note that the augmented running reward rate $R$, introduced in
\eqref{eq:augmented-reward} for theoretical analysis, cannot be
evaluated from sample trajectories because its jump-reward
component depends on the unknown transition rates. Accordingly, for algorithm design, the martingale condition in
Theorem~\ref{thm:optimal-J-q} is formulated in terms of the cumulative
running and jump rewards, which are directly observable along sample
trajectories.

The two parts of Theorem~\ref{thm:optimal-J-q} provide the necessity and
sufficiency underlying our learning procedure. 
Part~(i) shows that $(J^*,q^*)$ satisfies the martingale condition, which provides the foundation for designing $q$-learning algorithm as we will see later. 
 Part~(ii) shows that a single policy and a fixed decision grid on the entire horizon $[0,T]$ suffices to
uniquely identify $(J^*,q^*)$, provided that the martingale condition
holds for every initial state.

Theorem~\ref{thm:optimal-J-q} differs from existing results established in diffusion and jump-diffusion settings \citep{jia2023qlearning, gao2026reinforcement} in two main aspects. First, our result incorporates the jump rewards inherent to CTJMDPs (such as the network dynamic pricing application considered later), which are not considered in these prior studies. Second, Part~(ii) requires significantly less grid coverage than the corresponding results in \citep{jia2023qlearning, gao2026reinforcement}.
Their formulation requires the martingale condition to hold for every initial time-state pair $(t,x)$ and every possible grid on $[t,T]$. In contrast, our condition is far less restrictive, requiring only one fixed grid on $[0,T]$ across all initial states. This distinction stems directly from the path properties of the underlying processes. The grid sample state process of a CTJMDP remains constant between jump events, whereas the sample state of a diffusion process evolves continuously. In our framework, under the stability condition $\lambda^*(x)<\infty$, the pure-jump grid sample state process initialized at $x$ has a strictly positive probability of remaining in state $x$ up to any arbitrary time in $[0,T]$. This property is the key mechanism that allows us to establish Part~(ii) under relaxed grid conditions.

\subsection{$q$-learning Algorithm for CTJMDPs}
We next develop the $q$-learning
algorithm for the entropy-regularized RL problem
\eqref{eq:value-function-J-given-pi}--\eqref{eq:optimal-value-function-J}.
The martingale characterization in Theorem \ref{thm:optimal-J-q} motivates
the simultaneous approximation of the optimal value function and $q$-function.
We consider parametric families
$\{J^\theta:\theta\in\Theta\}$ and
$\{q^\psi:\psi\in\Psi\}$ satisfying
\begin{align}\label{eq:approx-condition}
    J^\theta(T,x)
    &=h(x),\qquad
    \int_{\mathcal{A}}
    \exp\{\frac{1}{\gamma}q^\psi(t,x,a)\}
    \mu(\mathrm{d}a)
    =1,
    \qquad (t,x)\in[0,T]\times\mathcal{X}.
\end{align}
The normalization condition in \eqref{eq:approx-condition} ensures that $\bpi^\psi(a\mid t,x) \triangleq \exp\{\frac{1}{\gamma}q^\psi(t,x,a)\}$
defines a policy density with respect to $\mu$. If the state space is not a subset of $\mathbb{R}^d$, we can map $x \in \mathcal{X}$ to $x' \in \mathbb{R}^d$ before passing it to the parametric functions (e.g. neural nets) $J^{\theta}$ and $q^{\psi}$.
For implementation, we fix $\Delta t>0$, let
$K\triangleq\lceil T/\Delta t\rceil$ and $t_k\triangleq k\Delta t$ for $k=0,\ldots,K-1$, and define $\mathbb{S} \triangleq \{0=t_0<t_1<\cdots<t_K=T\}$. Note that one can also consider a non-uniform time grid. 

To design the updating rules for \(\theta\) and \(\psi\) based on Theorem~\ref{thm:optimal-J-q}, we utilize the martingale orthogonality condition, which states that a process \(M\) is a (square-integrable) \(\{\mathcal{F}_t\}_{t\in [0, T]}\)-martingale if and only if
\(
\mathbb{E}[\int_0^T H_t\, \mathrm{d}M_t] = 0
\)
holds for any predictable process \(H\) (called a test function). A popular approach is to choose
\begin{align}\label{eq:choice-xi-zeta}
\xi_t = \nabla_\theta J^\theta(t, X_{t-}^{\bpi^\psi,\mathbb{S}}), \qquad
\zeta_t = \nabla_\psi q^\psi(t, X_{t-}^{\bpi^\psi,\mathbb{S}}, a_t^{\bpi^\psi,\mathbb{S}})
\end{align}
as two sets of test functions. The martingale orthogonality conditions then reduce to
\begin{align}\label{eq:system-equation}
    \mathbb{E}\left[
        \int_0^T \xi_t\,
        \mathrm{d}M_t^{\theta,\psi}
    \right]
    =0, \qquad
    \mathbb{E}\left[
        \int_0^T \zeta_t\,
        \mathrm{d}M_t^{\theta,\psi}
    \right]
    =0,
\end{align}
where
\begin{align*}
    \mathrm{d} M_t^{\theta,\psi} \triangleq \mathrm{d}J^\theta(t,X_t^{\bpi^\psi,\mathbb{S}}) &+ \sum_{y\in\mathcal X}\sum_{z\neq y} \rho\bigl(t,y,a_t^{\bpi^\psi,\mathbb{S}},z\bigr)\,\mathrm{d}N_{yz}^{\bpi^\psi,\mathbb{S}}(t)
    \\
    &+ [r(t,X_{t-}^{\bpi^\psi,\mathbb{S}},a_t^{\bpi^\psi,\mathbb{S}})-q^\psi(t,X_{t-}^{\bpi^\psi,\mathbb{S}},a_t^{\bpi^\psi,\mathbb{S}})]\mathrm{d}t. 
\end{align*}
In \eqref{eq:choice-xi-zeta},  $\xi_t$ is chosen so that the update of $\theta$ corresponds to TD(0)
in discrete-time RL, while $\zeta_t$ is chosen because the identity
$\bpi^\psi(a\mid t,x) = \exp\{\frac{1}{\gamma}q^\psi(t,x,a)\}$ yields
$\zeta_t=\gamma\nabla_\psi\log\bpi^\psi(a_t^{\bpi^\psi,\mathbb{S}}\mid t,X_{t-}^{\bpi^\psi,\mathbb{S}})$, connecting
the update of $\psi$ to policy-gradient methods.
Note that the choices of $\xi_t$ and $\zeta_t$ are not
restricted to those specified in \eqref{eq:choice-xi-zeta}. Many alternative predictable test processes can be used, with different choices generally leading to
different learning algorithms; See \cite{jia2023qlearning} for details. 
We solve the system of equations in \eqref{eq:system-equation} by stochastic approximation and the update rules for $\theta$ and $\psi$ are given by
\begin{align}\label{eq:updating-rule}
    \theta
    \leftarrow
    \theta
    +
    \alpha_\theta
    \int_0^T
    \xi_t\,\mathrm{d}M_t^{\theta,\psi},
    \qquad
    \psi
    \leftarrow
    \psi
    +
    \alpha_\psi
    \int_0^T
    \zeta_t\,\mathrm{d}M_t^{\theta,\psi},
\end{align}
where $\alpha_\theta$ and $\alpha_\psi$ are the corresponding learning
rates.

We next discuss the computation of the integral in the update \eqref{eq:updating-rule}.
Unlike the controlled diffusion processes in
\citet{jia2023qlearning}, we exploit the piecewise-constant sample paths of
the state and action processes in our CTJMDP setting. Fix a realization
$\{(x_t,a_t): t\in [0, T]\}$ with state jump times 
$0<\tau_1<\cdots<\tau_L\leq T$. For each $\ell=1,\ldots,L$, let $\rho_\ell$
be the jump reward observed at time $\tau_\ell$.
We augment the decision grid with these realized jump times and write
\begin{align*}
    \mathbb{S}\cup\{\tau_1,\ldots,\tau_L\}
    =
    \{0=s_0<s_1<\cdots<s_M=T\},
\end{align*}
where duplicate points are removed. For each $m=0,\ldots,M-1$, let $x_m$ and $a_m$ denote the constant
values of the state and action processes on $(s_m,s_{m+1})$, respectively.
Let $\Delta C_m$ denote the running reward accumulated over this interval,
and set $\Delta_m\triangleq s_{m+1}-s_m$.
Define
\begin{align*}
    H^\theta(t,x,a)
    &\triangleq \nabla_\theta J^\theta(t,x),\qquad
    H^\psi(t,x,a)
    \triangleq \nabla_\psi q^\psi(t,x,a).
\end{align*}
Then, for $v\in\{\theta,\psi\}$, the corresponding pathwise increment in
\eqref{eq:updating-rule} is approximated by
\begin{align}
\mathcal{I}_{v} & \triangleq \int_0^T
H^v(t,x_{t-},a_t)
\mathrm{d}J^\theta(t,x_t) +
\sum_{y\in\mathcal X}\sum_{z\neq y} \int_0^T H^v(t,x_{t-},a_{t})  \rho(t,y,a_t,z)\,\mathrm{d}N_{yz}(t)\\
&\hspace{5cm}+ \int_0^T
H^v(t,x_{t-},a_t)[
    r(t,x_{t-},a_t)
    -q^\psi(t,x_{t-},a_t)]\mathrm{d}t
\nonumber\\
&\approx 
\sum_{\ell=1}^{L}
H^v(\tau_\ell,x_{\tau_\ell-},a_{\tau_\ell})
[J^\theta(\tau_\ell,x_{\tau_\ell}) - J^\theta(\tau_\ell,x_{\tau_\ell-}) + \rho_{\tau_l}]
\nonumber\\
&\hspace{1em}+
\sum_{m=0}^{M-1}
H^v(s_m,x_m,a_m)
[
    J^\theta(s_{m+1},x_m)
    -J^\theta(s_m,x_m) 
    +\Delta C_m
    -q^\psi(s_m,x_m,a_m)\Delta_m
].
\label{eq:integral-approximation}
\end{align}
This construction incorporates every realized state jump and grid-based
action change and uses the observed cumulative reward without temporal
quadrature. Consequently, the remaining approximation error
arises only from the time variation of the integrands within each interval.
With all the above in
place, we present our $q$-learning algorithm in Algorithm~\ref{alg:q-learning}.

\begin{remark}
Although our \(q\)-learning algorithm requires a decision grid to implement the randomized policy, its role differs fundamentally from conventional approaches that first discretize time and then apply discrete-time RL algorithms. The latter typically rely on a fixed, uniform grid throughout training, rendering their learning targets grid-dependent; while a finer grid provides a better continuous-time approximation, it significantly increases the computational overhead. By contrast, our learning targets \((J^*,q^*)\) are inherently grid-independent, continuous-time objectives. Consequently, our algorithmic implementation accommodates flexible, non-uniform time discretizations and seamlessly incorporates the random transition times of CTJMDPs from the collected data. This reduces discretization errors and yields more efficient RL algorithms.

\end{remark}

\begin{algorithm}[h]
\caption{Episodic $q$-Learning Algorithm}
\label{alg:q-learning}
\begin{algorithmic}[1]

\Require Initial state $x_0$, horizon $T$, decision grid
$\mathbb{S}=\{0=t_0<t_1<\cdots<t_K=T\}$, number of episodes $N$,
temperature parameter $\gamma$, and learning rates
$\alpha_{\theta}$ and
$\alpha_{\psi}$. 
Approximators $J^\theta$ and $q^\psi$ satisfying
\eqref{eq:approx-condition}. 
Policy $\bpi^\psi(a\mid t,x) \triangleq \exp\{\frac{1}{\gamma}q^\psi(t,x,a)\}$

\State \textbf{Simulator:} Given
$(t_k,t_{k+1},X_{t_k},A_k)$, simulate the system over
$(t_k,t_{k+1}]$ 
and return the state at $t_{k+1}$, all state jumps and the
corresponding jump rewards, and the cumulative running reward between successive grid or jump times

\State Initialize $\theta_1$ and $\psi_1$

\For{$j=1,\ldots,N$}

    \State Set $X_0\gets x_0$, $C_0\gets0$, and
    $\mathcal{T}_{J}\gets\varnothing$

    \For{$k=0,\ldots,K-1$}

        \State Sample
        $A_k\sim\bpi^{\psi_j}(\cdot\mid t_k,X_{t_k})$
        and set $A_t\gets A_k$ for $t\in(t_k,t_{k+1}]$

        \State Simulate the system over $(t_k,t_{k+1}]$ and observe
        $X_{t_{k+1}}$

        \State Record all state jump times $\tau_{\ell}$ and post-jump states in
        $(t_k,t_{k+1}]$, and add the jump times to $\mathcal{T}_{J}$

        \State Record the cumulative running reward $\Delta C_m$ between
        successive grid or jump times in $[t_k,t_{k+1}]$
        \State Record the jump reward $\rho_\ell$ observed at each jump time
        $\tau_\ell$ in $(t_k,t_{k+1}]$

    \EndFor

    \State Let
    $\mathcal{T}_{J}=\{\tau_1<\cdots<\tau_L\}$ and construct
    \[
        \mathbb{S}\cup\mathcal{T}_{J}
        =
        \{0=s_0<s_1<\cdots<s_M=T\},
    \]
    with duplicate points removed

    \State For each $m=0,\ldots,M-1$, identify the constant state and
    action $(x_m,a_m)$ on $(s_m,s_{m+1})$ and set
    \[
        \Delta_m\gets s_{m+1}-s_m
    \]

    \State Evaluate
    $\mathcal{I}_{\theta}^{(j)}$ and
    $\mathcal{I}_{\psi}^{(j)}$ from
    \eqref{eq:integral-approximation} using
    $(\theta_j,\psi_j)$ and the recorded trajectory

    \State Update
    \[
        \theta_{j+1}
        \gets
        \theta_j+\alpha_{\theta} \cdot \mathcal{I}_{\theta}^{(j)},
        \qquad
        \psi_{j+1}
        \gets
        \psi_j+\alpha_{\psi} \cdot \mathcal{I}_{\psi}^{(j)}
    \]

\EndFor

\end{algorithmic}
\end{algorithm}

\section{Case Study: Network Dynamic Pricing}\label{sec:dynamic-pricing}

In this section, we illustrate the application of our proposed RL framework for CTJMDPs. 
We consider the multi-product network dynamic pricing problem studied in \cite{gallego1997multiproduct} in the context of an airline network.
Let $\mathcal{I} \triangleq \{1, 2, \ldots, m\}$ denote the set of flight legs and $\mathcal{J} \triangleq \{1, 2, \ldots, n\}$ the set of origin-destination itineraries offered over a finite booking horizon $[0,T]$. Each flight leg $i\in\mathcal{I}$ has an initial capacity $c_i$, and each itinerary $j\in\mathcal{J}$ follows a fixed path and consumes one unit of capacity on every leg along that path. At any time, the decision maker observes the remaining leg capacities and sets the itinerary prices. 
Lower prices generate demand more rapidly but consume capacity that may be valuable for other itineraries, whereas higher prices preserve capacity at the cost of reduced demand. Because itineraries share flight legs, their pricing decisions are coupled through the remaining capacities.

Let $A=[a_{ij}]_{m \times n}$ be the leg-itinerary incidence matrix, where $a_{ij}=1$ if itinerary $j$ uses leg $i$ and $a_{ij}=0$ otherwise. We model the remaining leg capacities as the state $x=(x_1, x_2, \ldots, x_m)^{\top}$, with state space $\mathcal{X}=\prod_{i=1}^m\{0,1,\ldots,c_i\}.$
The action is a price vector $\boldsymbol{p}=(p_1, p_2, \ldots, p_n)^{\top}\in\mathcal{A}$, where $\mathcal{A}\subseteq\mathbb{R}_+^{n}$ denotes the set of allowable prices. Requests for itinerary $j$ arrive according to a Poisson process with rate $\lambda_j(\boldsymbol{p})$. Let $A_j$ denote column $j$ of $A$. A sale of itinerary $j$ changes the state from $x$ to $x-A_j$ and can occur only if every leg on the itinerary has positive remaining capacity. Thus, the transition rates are
\begin{align*}
    \lambda(x-A_j\mid t,x,\boldsymbol{p})
    =\lambda_j(\boldsymbol{p})\boldsymbol{1}_{\{x\geq A_j\}},
    \qquad j\in\mathcal{J}, 
\end{align*}
and $\lambda(x\mid t,x,\boldsymbol{p}) =
-\sum_{j=1}^n
\lambda_j(\boldsymbol{p})
\boldsymbol{1}_{\{x\geq A_j\}}$. 
Since revenue is earned only upon a sale, the running reward rate and the terminal reward are zero.
The jump reward function is given by
\[
    \rho(t,x,\boldsymbol{p},x-A_j)=p_j,
    \qquad j\in\mathcal J,\quad x\geq A_j.
\]
In the RL setting, the decision maker does not know the demand functions but observes the realized sales, the corresponding revenues, and the resulting changes in the remaining leg capacities.

For both numerical instances considered below, the booking horizon is
normalized to $T=1$. Requests for different itineraries arrive according
to independent Poisson processes, and demand is time homogeneous and
separable across itineraries. Specifically, the demand function for itinerary $j$ is given by
\begin{align}\label{eq:demand-function}
    \lambda_j(\boldsymbol{p}) \equiv \lambda_j(p_j)
    =
    \lambda_j^0 \exp\{-\epsilon_j^0(\frac{p_j}{p_j^0} -1)\},
\end{align}
where $\lambda_j^0$, $p_j^0$, and $\epsilon_j^0$ denote the reference demand
rate, reference price, and price elasticity, respectively. Under the above formulation, one can readily check that Assumption~\ref{asp:transtion-reward-rate} is satisfied. In particular, since the transition rates are conservative in the sense that $\sum_{y \in \calX} \lambda(y \mid t, x, a) \equiv 0$, and the running reward rates and terminal reward are zero, one may take $w(x)\equiv1$, $c>0$, $b=0$, and $M_1=
\sum_{j=1}^n \lambda_j^0p_j^0 
e^{\epsilon_j^0-1} / \epsilon_j^0$.
\paragraph{Benchmarks.} We consider the following benchmarks for comparison.
\begin{itemize}
    \item \textit{Dynamic programming (DP).} 
    The DP benchmark is obtained from a time-discretized approximation of the
    continuous-time pricing model. Specifically, given
    $\Delta t'>0$, let $K' \triangleq \lceil \frac{T}{\Delta t'} \rceil$ and
    $t_k'=k\Delta t'$.
    In each interval $[t_k',t_{k+1}')$, the price vector
    $\boldsymbol p$ is held fixed. Conditional on the state being $x$ at
    $t_k'$, a sale of itinerary $j$ occurs with probability
    $\lambda_j(p_j)\boldsymbol{1}_{\{x\geq A_j\}}\Delta s$, and no sale occurs
    with the remaining probability. The optimal value function
    $V_{\Delta t'}^*$ of the resulting discrete-time model satisfies
    \begin{align*}
    &V_{\Delta t'}^*(t_k',x)
    &&=&&V_{\Delta t'}^*(t_{k+1}',x)
    +\Delta t'
    \max_{\boldsymbol p\in\mathbb{R}_+^n}
    \sum_{j=1}^n
    \lambda_j(p_j)\boldsymbol{1}_{\{x\geq A_j\}}
    \big[
    p_j+V_{\Delta t'}^*(t_{k+1}',x-A_j)
         -V_{\Delta t'}^*(t_{k+1}',x)
    \big],\\
    &V_{\Delta t'}^*(T,x)&&=&& 0.
    \end{align*}
    For separable exponential demand functions of the form in \eqref{eq:demand-function}, 
    the optimal price is given by 
    \[
    p_{j,k}^*(x)
    =
    V_{\Delta t'}^*(t_{k+1}',x)
    -
    V_{\Delta t'}^*(t_{k+1}',x-A_j)+\frac{p_j^0}{\epsilon_j^0}.
    \]
    Hence, no discretization of the price space is required. For sufficiently
    small $\Delta t'$, $V_{\Delta t'}^*(0,c)$ provides a numerical
    approximation to $V^*(0,c)$, the optimal value of the original
    continuous-time pricing problem.

    \item \textit{Deterministic upper bound (UB).}
    The deterministic problem replaces stochastic sales with their expected
    rates and solves
    \begin{align*}
        \sup_{\boldsymbol p(\cdot)}
        \quad
        \int_0^T
        \boldsymbol{p}(t)^{\top}\boldsymbol{\lambda}\bigl(\boldsymbol{p}(t)\bigr)\,\mathrm dt
        \qquad \text{s.t.}\quad
        \int_0^T
        A\boldsymbol\lambda(\boldsymbol p(t))\,\mathrm dt
        \leq c,
    \end{align*}
    where
    \(
        \boldsymbol\lambda(\boldsymbol p)
        =
        \bigl(
        \lambda_1(p_1),\ldots,\lambda_n(p_n)
        \bigr)^\top.
    \)
    Its optimal value provides an upper bound on the optimal expected
    revenue of the stochastic problem.

    \item \textit{Fluid pricing (FP).}
    The FP policy follows the optimal price path obtained from the
    deterministic problem. An itinerary is made unavailable once any of its required legs no longer has sufficient remaining capacity.

    \item \textit{Fluid pricing with booking limits (FP-BL).}
    The FP-BL policy follows the same deterministic price path as the FP policy
    but imposes a booking limit for each itinerary. Specifically, under an optimal deterministic price path $\boldsymbol p^d(t)$, the booking limit for itinerary $j$ is set to
    \(
        \lfloor
        \int_0^T
        \lambda_j\bigl(p_j^d(t)\bigr)\,\mathrm dt
        \rfloor
    \).
    Itinerary $j$ is made unavailable once its booking limit is reached.
\end{itemize}
The benchmarks above require knowledge of the demand functions, whereas
Algorithm~\ref{alg:q-learning} uses only observed state transitions and
revenues. The FP and FP-BL policies are heuristics constructed from the
optimal solution of the deterministic problem.
\citet{gallego1997multiproduct} shows that both policies are asymptotically
optimal when the initial capacities and demand rates increase
proportionally, and their revenues accordingly approach the deterministic
upper bound as the scale increases.

\paragraph{Implementation of Algorithm~\ref{alg:q-learning}.}
For both instances, we employ two separate fully connected neural networks,
$f_{\mathrm{critic}}^\theta$ and $f_{\mathrm{actor}}^\psi$, to construct
the approximators $J^\theta$ and $q^\psi$, respectively. Both networks use ReLU activations in the hidden layers, whose number and width are specified for each instance below. The parameters of each network are updated using a
separate Adam optimizer. Both networks take the normalized
$(m+1)$-dimensional time-state vector
\(
    (\tilde t,\tilde x)
    =
    (
        1-\frac{t}{T},
        \frac{x_1}{c_1},\ldots,\frac{x_m}{c_m}
    )
\)
as input. The critic network outputs a scalar, and the value-function
approximator is defined by
\begin{align*}
    J^\theta(t,x)
    =
    (1-\frac{t}{T}) \cdot 
    f_{\mathrm{critic}}^\theta(\tilde t,\tilde x).
\end{align*}
This construction ensures
that $J^\theta(T,x)=0=h(x)$ for every $x\in\mathcal X$. The actor network outputs two positive
$n$-dimensional vectors,
$\mu^\psi(\tilde t,\tilde x)$ and
$k^\psi(\tilde t,\tilde x)$. Conditional on $(t,x)$, the itinerary
prices are sampled independently according to
\begin{align}
\label{eq:gamma-parametrization}
    p_j
    \sim
    \operatorname{Gamma}
    \left(
        \frac{k_j^\psi(\tilde t,\tilde x)}{\gamma},
        \frac{k_j^\psi(\tilde t,\tilde x)}
        {\gamma\mu_j^\psi(\tilde t,\tilde x)}
    \right),
    \qquad j=1,\ldots,n,
\end{align}
where the Gamma distribution is parameterized by its shape and rate.
We choose the Gamma distribution because its support on $[0,\infty)$
ensures that sampled prices are nonnegative, while its two parameters allow
flexible control over the mean and relative dispersion. Under
\eqref{eq:gamma-parametrization}, the mean and variance are given by
$\mu_j^\psi$ and
$\gamma(\mu_j^\psi)^2/k_j^\psi$, respectively. Thus, a smaller temperature
parameter $\gamma$ yields a policy that is more concentrated around its
mean, consistent with the behavior of the optimal policy characterized in
\eqref{eq:optimal-policy}.
Let $\bpi_j^\psi(\cdot\mid t,x)$ denote the resulting marginal price density
for itinerary $j$. The joint policy density and $q$-function approximator are
defined by
\begin{align*}
    \bpi^\psi(\boldsymbol p\mid t,x)
    =
    \prod_{j=1}^n
    \bpi_j^\psi(p_j\mid t,x),
    \qquad
    q^\psi(t,x,\boldsymbol p)
    =
    \gamma
    \sum_{j=1}^n
    \log\bpi_j^\psi(p_j\mid t,x).
\end{align*}
This construction ensures that $q^\psi$ satisfies the normalization
condition in \eqref{eq:approx-condition}.

\subsection{A Small-Network Example}\label{ssec:small-pricing-example}
We first consider a two-leg network with three itineraries. The first two itineraries each use a single leg, whereas the third uses both legs. The initial capacity vector and the leg--itinerary incidence matrix are given by $c=(12,10)^\top$ and $A=
\begin{bmatrix}
    1&0&1\\
    0&1&1
\end{bmatrix}$.
The demand parameters are reported in
Table~\ref{tab:small-pricing-parameters}.
\begin{table}[htbp]
    \centering
    \caption{Itinerary and demand parameters for the example in Section~\ref{ssec:small-pricing-example}}
    \label{tab:small-pricing-parameters}
    \small
    \begin{tabular}{@{}cccc@{}}
        \toprule
        Itinerary $j$
        & $\lambda_j^0$
        & $\epsilon_j^0$
        & $p_j^0$\\
        \midrule
        1 & 30 & 4.0 & 20\\
        2 & 15 & 4.0 & 30\\
        3 & 50 & 4.0 & 45\\
        \bottomrule
    \end{tabular}
\end{table}

To implement Algorithm~\ref{alg:q-learning} for this instance, we configure
both the actor and critic networks with two hidden layers, each of width 32. 
We set the initial state to $x_0=c$ and the time-grid spacing to
$\Delta t=0.01$. The actor and critic
learning rates are $\alpha_\psi=3\times10^{-5}$ and
$\alpha_\theta=5\times10^{-4}$, respectively, and the temperature
parameter is $\gamma=1\times 10^{-2}$.

During training, we periodically evaluate the current learned policy by averaging
the total revenue over 10,000 independent simulation runs.
Figure~\ref{fig:small-pricing-example} reports the average revenues of the
policies generated by Algorithm~\ref{alg:q-learning} throughout the
learning process, together with reference lines for the
deterministic upper bound, the DP optimal value, and the revenues of FP
and FP-BL. 
Here, the DP optimal value is computed with a time-discretization step of \(\Delta s = 10^{-4}\), yielding \(V_{0.0001}^*(0,c) = 709.883\).
Table~\ref{tab:small-pricing-example} reports the performance of
Algorithm~\ref{alg:q-learning}, FP, and FP-BL relative to the DP optimal
value.
The learned policy outperforms FP and FP-BL by $8.87\%$ and $18.91\%$,
respectively, while attaining performance within $2.39\%$ of the DP
optimal value. These results demonstrate that the proposed 
algorithm can learn a near-optimal pricing policy.
Note that although the FP and FP-BL policies are asymptotically optimal, they exhibit
performance gaps of $10.33\%$ and $17.90\%$, respectively, relative to the
DP optimal value in this instance.
These gaps arise from the small capacities and resource coupling in this
instance, which make the opportunity costs of capacity sensitive to
realized demand.
FP retains fixed prices, whereas FP-BL additionally imposes fixed
itinerary-level allocations, so neither adapts these decisions to the evolving system state. In contrast, our algorithm learns
a state-dependent pricing policy that responds to the remaining time and
capacities, thereby achieving performance close to the DP optimal value.

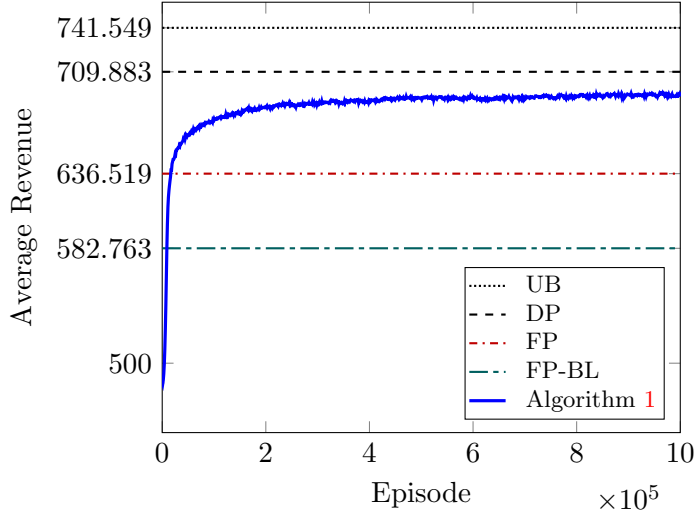
\begin{figure}
\centering
\caption{Average revenue of Algorithm \ref{alg:q-learning} over episodes for the
example in Section~\ref{ssec:small-pricing-example}}
\label{fig:small-pricing-example}
\begin{tikzpicture}
    \begin{axis}[
        xlabel={Episode},
        ylabel={Average Revenue},
        xtick={0,200000,400000,600000,800000,1000000},
        xticklabels={0,2,4,6,8,10},
        xtick scale label code/.code={$\times 10^5$},
        ytick={400,500,582.763, 636.519, 709.883, 741.549},
        yticklabel style={
            /pgf/number format/fixed,
            /pgf/number format/precision=3,
            /pgf/number format/1000 sep={,}
        },
        scaled y ticks=false,
        xmin=0, xmax=1000000,
        ymin=450, ymax=760,
        legend pos=south east,
        legend style={
            font=\footnotesize,
            inner sep=1pt,
            row sep=-1pt,
            column sep=5pt,
            cells={anchor=west}
        },
        legend image post style={xscale=0.8},
        grid=minor
    ]

    \addplot[
        thick,
        black,
        densely dotted
    ]
    coordinates {(0,741.549) (1000000,741.549)};
    \addlegendentry{UB}
    
    \addplot[
        thick,
        black,
        dashed
    ]
    coordinates {(0,709.883) (1000000,709.883)};
    \addlegendentry{DP}
    
    \addplot[
        thick,
        red!75!black,
        dashdotted
    ]
    coordinates {(0,636.519) (1000000,636.519)};
    \addlegendentry{FP}
    
    \addplot[
        thick,
        teal!75!black,
        dash pattern=on 7pt off 2pt on 2pt off 2pt
    ]
    coordinates {(0,582.763) (1000000,582.763)};
    \addlegendentry{FP-BL}
    
    \addplot[
        very thick,
        blue,
        solid
    ]
    table[col sep=comma, x=x, y=y]
        {data/small_pricing_example.csv};
    \addlegendentry{Algorithm~\ref{alg:q-learning}}
    \end{axis}
\end{tikzpicture}
\end{figure}
\begin{table}[htbp]
    \centering
    \caption{Numerical results for the example in Section~\ref{ssec:small-pricing-example}}
    \label{tab:small-pricing-example}
    \small
    \begin{tabular}{@{}lccr@{}}
        \toprule
        Method
        & Average revenue
        & 99\%-CI ($\pm$)
        & Gap from DP\\
        \midrule
        DP
        & 709.883
        & --
        & 0.00\%\\
        Algorithm~\ref{alg:q-learning}
        & 692.951
        & 0.811
        & 2.39\%\\
        FP
        & 636.519
        & 2.599
        & 10.33\%\\
        FP-BL
        & 582.763
        & 2.392
        & 17.90\%\\
        \bottomrule
    \end{tabular}
\end{table}
\subsection{A Large-Network Example}\label{ssec:large-pricing-example}
Next, we evaluate Algorithm \ref{alg:q-learning} on an airline pricing problem
with a large state space and a high-dimensional continuous action space. The
instance is based on Example 1 of \citet{gallego1997multiproduct}, which
comprises a six-node network with 11 flight legs and 18 itineraries.
Figure \ref{fig:pricing-network} depicts the network and the initial leg capacities.
The reference demand rate is set to
$\lambda_j^0=30$ for every itinerary, and the remaining demand parameters
and itinerary paths are reported in Table
\ref{tab:pricing-parameters}.
This instance corresponds to the scale-$0.1$ version of Example~1 in
\citet{gallego1997multiproduct}, where scale-$\kappa$ means that both the
initial leg capacities and reference demand rates are set to $\kappa$ times
their values in the base instance, with all other model parameters unchanged.
We focus on the scale-0.1 instance
because the FP-BL and FP policies are asymptotically optimal as the scale $\kappa$ grows, so we use a small-$\kappa$ instance to demonstrate the advantages of our algorithm in a non-asymptotic
regime.
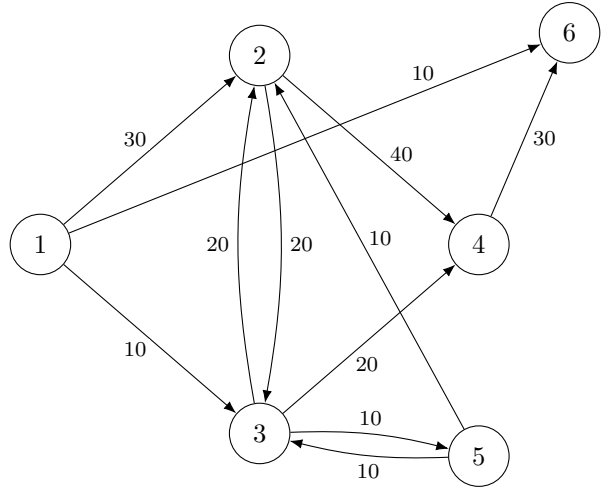
\begin{figure}
    \centering
    \caption{Airline network for the example in Section~\ref{ssec:large-pricing-example}}
    \label{fig:pricing-network}
    \begin{tikzpicture}[
        >=Latex,
        airport/.style={
            circle,
            draw,
            minimum size=8mm,
            inner sep=0pt,
            font=\small
        },
        leg/.style={
            ->,
            thin
        },
        capacity/.style={
            fill=white,
            inner sep=1.2pt,
            font=\scriptsize
        }
    ]
        \node[airport] (n1) at (0,0) {1};
        \node[airport] (n2) at (2.9,2.5) {2};
        \node[airport] (n3) at (2.9,-2.5) {3};
        \node[airport] (n4) at (5.8,0) {4};
        \node[airport] (n5) at (5.8,-2.8) {5};
        \node[airport] (n6) at (7.0,2.8) {6};

        \draw[leg]
            (n1) -- node[capacity,above left] {30} (n2);
        \draw[leg]
            (n1) -- node[capacity,below left] {10} (n3);
        \draw[leg]
            (n1) --node[capacity,above,pos=0.75,yshift=3pt]{10} (n6);

        \draw[leg]
            (n2) to[bend left=10]
            node[capacity,right,xshift=2pt] {20} (n3);
        \draw[leg]
            (n3) to[bend left=10]
            node[capacity,left, xshift=-2pt] {20} (n2);

        \draw[leg]
            (n2) -- node[capacity,above right,pos=0.6] {40} (n4);
        \draw[leg]
            (n3) -- node[capacity,below right, pos=0.4] {20} (n4);
        \draw[leg]
            (n5) -- node[capacity,right,pos=0.55, xshift=2pt] {10} (n2);

        \draw[leg]
            (n3) to[bend left=8]
            node[capacity,above,yshift=2pt] {10} (n5);
        \draw[leg]
            (n5) to[bend left=8]
            node[capacity,below,yshift=-2pt] {10} (n3);

        \draw[leg]
            (n4) -- node[capacity,right, xshift=2pt] {30} (n6);
    \end{tikzpicture}
\end{figure}
\begin{table}
    \centering
    \caption{Itinerary and demand parameters for example in Section~\ref{ssec:large-pricing-example}}
    \label{tab:pricing-parameters}
    \small
    \begin{tabular}{@{}ccccc@{}}
        \toprule
        Itinerary $j$
        & O--D pair
        & Path
        & $\epsilon_j^0$
        & $p_j^0$ \\
        \midrule
         1 & $1$--$2$ & $1$--$2$             & 1.0 & 220 \\
         2 & $1$--$3$ & $1$--$3$             & 1.2 & 220 \\
         3 & $1$--$4$ & $1$--$2$--$4$        & 2.0 & 400 \\
         4 & $1$--$5$ & $1$--$3$--$5$        & 1.0 & 250 \\
         5 & $1$--$6$ & $1$--$6$             & 0.8 & 200 \\
         6 & $2$--$3$ & $2$--$3$             & 1.0 & 230 \\
         7 & $2$--$4$ & $2$--$4$             & 0.9 & 200 \\
         8 & $2$--$5$ & $2$--$3$--$5$        & 2.0 & 200 \\
         9 & $2$--$6$ & $2$--$4$--$6$        & 1.0 & 200 \\
        10 & $3$--$2$ & $3$--$2$             & 1.0 & 200 \\
        11 & $3$--$4$ & $3$--$4$             & 2.0 & 230 \\
        12 & $3$--$5$ & $3$--$5$             & 2.0 & 120 \\
        13 & $3$--$6$ & $3$--$4$--$6$        & 2.0 & 150 \\
        14 & $4$--$6$ & $4$--$6$             & 1.0 & 150 \\
        15 & $5$--$2$ & $5$--$2$             & 1.0 & 200 \\
        16 & $5$--$3$ & $5$--$3$             & 2.0 & 150 \\
        17 & $5$--$4$ & $5$--$3$--$4$        & 1.0 & 160 \\
        18 & $5$--$6$ & $5$--$3$--$4$--$6$   & 1.0 & 230 \\
        \bottomrule
    \end{tabular}
\end{table}

For this instance, we implement Algorithm~\ref{alg:q-learning} as follows. Both the actor and critic networks have two hidden layers of width 128. We initialize the state at \(x_0 = c\) and set the time-grid spacing to \(\Delta t = 0.01\). The learning rates are \(\alpha_{\theta} = 3 \times 10^{-6}\) for the actor and \(\alpha_{\phi} = 3 \times 10^{-4}\) for the critic, and the temperature parameter is set to \(\gamma = 1 \times 10^{-2}\).
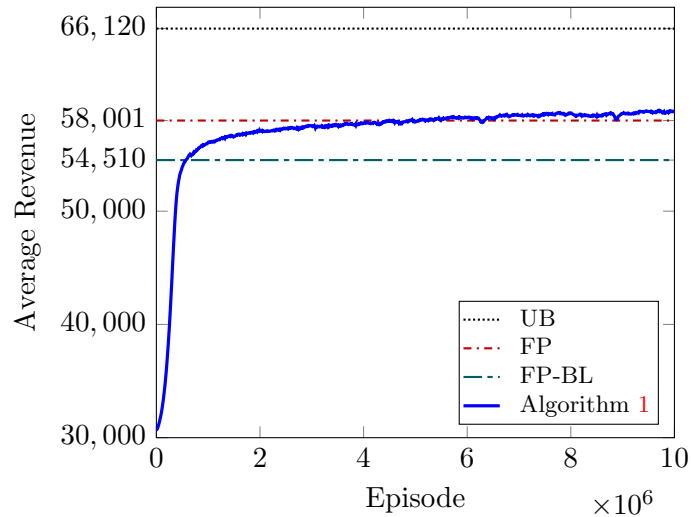
\begin{figure}
\centering
\caption{Average revenue of Algorithm \ref{alg:q-learning} over episodes for
the example in Section~\ref{ssec:large-pricing-example}}
\label{fig:large-pricing-example}
\begin{tikzpicture}
    \begin{axis}[
        xlabel={Episode},
        ylabel={Average Revenue},
        xtick={0,2000000,4000000,6000000,8000000,10000000},
        xticklabels={0,2,4,6,8,10},
        xtick scale label code/.code={$\times 10^6$},
        ytick={30000,40000,50000,54510,58001,66120},
        yticklabel style={
            /pgf/number format/fixed,
            /pgf/number format/precision=0,
            /pgf/number format/1000 sep={,}
        },
        scaled y ticks=false,
        xmin=0, xmax=10000000,
        ymin=30000, ymax=68000,
        legend pos=south east,
        legend style={
            font=\footnotesize,
            inner sep=1pt,
            row sep=-1pt,
            column sep=5pt,
            cells={anchor=west}
        },
        legend image post style={xscale=0.8},
        grid=minor
    ]
    \addplot[
        thick,
        black,
        densely dotted
    ]
    coordinates {(0,66120) (10000000,66120)};
    \addlegendentry{UB}
    
    \addplot[
        thick,
        red!75!black,
        dashdotted
    ]
    coordinates {(0,58001) (10000000,58001)};
    \addlegendentry{FP}
    
    \addplot[
        thick,
        teal!75!black,
        dash pattern=on 7pt off 2pt on 2pt off 2pt
    ]
    coordinates {(0,54510) (10000000,54510)};
    \addlegendentry{FP-BL}
    \addplot[very thick, blue, solid]
        table[col sep=comma, x=x, y=y]
        {data/large_pricing_example.csv};
    \addlegendentry{Algorithm \ref{alg:q-learning}}
    \end{axis}
\end{tikzpicture}
\end{figure}

The state space for this instance contains approximately \(5.88 \times 10^{13}\) states, rendering direct dynamic programming computationally infeasible. Figure \ref{fig:large-pricing-example} reports the average revenue of the
policies generated by Algorithm \ref{alg:q-learning} throughout the learning process,
together
with the deterministic upper bound and the FP and FP-BL benchmarks.
It shows that the final policy learned by our algorithm achieves of
$58{,}818$, exceeding the FP-BL revenue by $7.9\%$.
Moreover, despite operating without knowledge of the demand functions, our algorithm
delivers performance comparable to that of the FP policy, whose
construction requires a known demand model.

\section{Conclusion and Future Work}

In this paper, we establish the theoretical foundations for reinforcement learning in CTJMDPs with general discrete state spaces and propose efficient $q$-learning algorithms. We apply this general framework to network dynamic pricing and queueing control, where numerical experiments demonstrate the strong performance of our approach.

This work opens several avenues for future research. From a theoretical perspective, critical directions include establishing convergence analysis for the proposed $q$-learning algorithm and studying the underlying stochastic approximation procedures. 
From an applied perspective, it would be interesting to apply our general framework to other high-dimensional operations research applications.

\newpage

\bibliographystyle{chicago}
\bibliography{refs}

\newpage
\appendix

\section{Case Study: Dynamic Server Allocation for Queue Control}\label{sec:queueing_example}
In this section, we apply our RL framework to a classic finite-horizon continuous-time queueing control problem \citep[Chapter VII]{bremaud1981point}, which can be formulated as a finite-horizon CTJMDP. 

Spefically, we consider a service system operated over a finite horizon $[0,T]$, with capacity $C$ and $S$ homogeneous servers. Let $x$ denote the queue length, defined as the total number of customers in the system. At any time $t$, the decision maker observes $x$ and chooses the number $a$ of active servers. Accordingly, the state and action spaces are $\mathcal{X}=\{0,1,\ldots,C\}$ and $\mathcal{A}=\{0,1,\ldots,S\}$, respectively. Activating more servers reduces congestion but increases operating costs, and the optimal choice depends on both the queue length and the remaining time. The number of active servers may be changed at any time, including periods with no customer arrival or service completion. This is in contrast to SMDPs where actions are restricted to event epochs or jump times.

Customers arrive according to a nonhomogeneous Poisson process with rate $\lambda(t)$, and each busy server completes service at rate $\mu(t)$. Thus, the transition rates are given by
\begin{align*}
\lambda(x+1\mid t,x,a)
&=\lambda(t)\boldsymbol{1}_{\{x<C\}}, \qquad
\lambda(x-1\mid t,x,a)
=(x\wedge a)\mu(t)\boldsymbol{1}_{\{x>0\}},
\end{align*}
and $\lambda(x\mid t,x,a)=-\lambda(t)\boldsymbol{1}_{\{x<C\}} -(x\wedge a)\mu(t)\boldsymbol{1}_{\{x>0\}}$. 
Arrivals that find the system at capacity are blocked. 
Let $K_1,\, K_2,\, K_3>0$ denote the operating cost per active server per unit time, the holding cost per customer per unit time, and the terminal cost per customer remaining at time $T$, respectively. That is, the running reward rate and the terminal reward are given as follows:
\begin{align*}
    r(t,x,a)=-K_1a-K_2x, \qquad h(x)=-K_3x.
\end{align*}
The jump reward is assume to be zero. 
These dynamics define a finite-horizon CTJMDP.
In the RL setting, the decision maker knows the cost coefficients $K_1$, $K_2$, and $K_3$, but not the arrival rate $\lambda(t)$ or the service rate $\mu(t)$. 
This finite-horizon queueing control problem, featuring
non-stationary optimal policies that depend on both system states and time, has not been explored from a model-free RL perspective to the best of our knowledge. We apply our proposed RL algorithm to learn the optimal policy.

To demonstrate the applicability of Algorithm \ref{alg:q-learning} to this problem, we consider a small instance with $T=10$ and $C=S=10$. The arrival and service rates are specified as $\lambda(t) = 1.5 +0.5 \sin(\frac{2\pi t}{T})$ and $\mu(t) = 0.5 + 0.1 \frac{t}{T}$. We set the cost coefficients to $K_1 = K_2 = K_3 = 1$.

For benchmarking purposes, we construct a discrete-time approximation of the continuous-time control problem and solve the resulting dynamic programming problem. Specifically, given a small $\Delta t' > 0$, let $K' \triangleq \lceil \frac{T}{\Delta t'} \rceil$, $t_k' \triangleq k \Delta t'$ for $k=0, \ldots, K'-1$, and $t_{K'} \triangleq T$.
In each period $[t_k', t_{k+1}')$, there is a probability of $\lambda(t_{k+1}')\boldsymbol{1}_{\{x \leq C-1\}} \Delta s$ that a customer arrivals, and a probability of $(x \wedge a) \mu(t_{k+1}')\boldsymbol{1}_{\{x \geq 1\}} \Delta s$ that a customer departs, given the queue length at $t_k'$ is $x$ and the number of active servers in this period is $a$. With the remaining probability, the system state does not change. 
The dynamic programming (DP) problem corresponding to the above discrete-time model is given by
\begin{align}\label{eq:DP_for_queueing_example}
        &V_{\Delta t'}^*(t_k', x) &&=&& V_{\Delta t'}^*(t_{k+1}', x) + \lambda(t_{k+1}') \boldsymbol{1}_{\{x \leq C-1\}} \Delta t'
        (V_{\Delta t'}^*(t_{k+1}', x+1) - V_{\Delta t'}^*(t_{k+1}', x) ) - K_2 x  \Delta t' \nonumber \\ 
        & && &&+\max_{a \in \calA} \{- K_1 a \Delta t' + (x \wedge a)\mu(t_{k+1}') \boldsymbol{1}_{\{x \geq 1\}} \Delta t' (V_{\Delta t'}^*(t_{k+1}', x-1) - V_{\Delta t'}^*(t_{k+1}', x))\}, \nonumber\\
        & V_{\Delta t'}^*(T, x) &&=&& - K_3 x. 
\end{align}
For $\Delta t'$ sufficiently small, the value $V_{\Delta t'}^*(0, 0)$ obtained from solving the DP problem \eqref{eq:DP_for_queueing_example} is expected to provide a reliable approximation to the optimal value $V^*(0, 0)$ under the continuous time model.
In this example, with $\Delta t'$ set to $0.0001$, we obtain $V^*_{0.0001}(0, 0) = -30.796$.

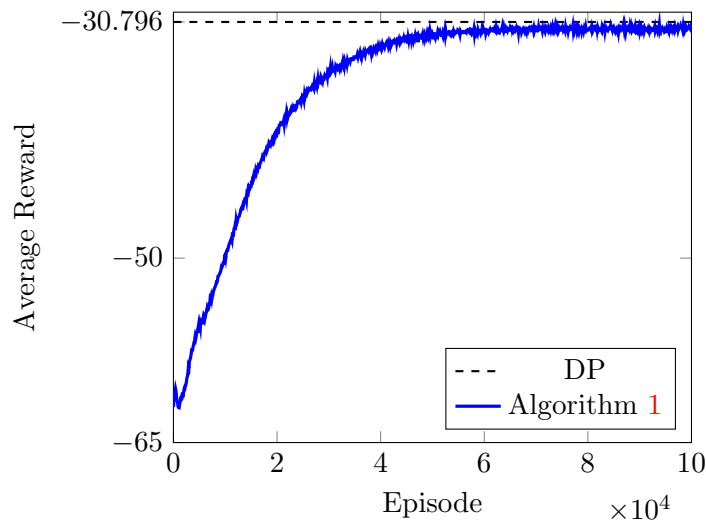
\begin{figure}[htbp]   
\centering
\caption{Average reward of Algorithm \ref{alg:q-learning} over episodes for queue control}
\label{fig:queueing_example}
\begin{tikzpicture}
    \begin{axis}[
        xlabel={Episode},
        ylabel={Average Reward},
        xtick={0,20000,40000,60000,80000,100000},
        xticklabels={0,2,4,6,8,10},
        xtick scale label code/.code={$\times 10^4$},
        ytick={-65, -50, -30.796},
        yticklabel style={
            /pgf/number format/fixed,
            /pgf/number format/precision=5
        },
        scaled y ticks=false,
        xmin=0, xmax=100000,
        ymin=-65, ymax=-30,
        legend pos=south east,
        grid=minor
    ]
    \addplot [
        thick,
        black,
        dashed
    ] coordinates {(0,-30.796) (100000,-30.796)};
    \addlegendentry{DP}
    
    \addplot[very thick, blue, solid] table [col sep=comma, x=x, y=y] {data/queueing_example.csv};
    \addlegendentry{Algorithm \ref{alg:q-learning}}
    \end{axis}
    \end{tikzpicture}
\end{figure}
To implement Algorithm \ref{alg:q-learning} for this instance, we employ two separate fully connected neural networks, $f_{\mathrm{critic}}^{\theta}$ and $f_{\mathrm{actor}}^{\psi}$, to construct the approximators $J^{\theta}$ and $q^{\psi}$, respectively. 
Both the critic and actor networks are configured with two hidden layers of width $8$ and use the ReLU activation
function.
The parameters of each network are updated using a separate Adam optimizer.
Both networks take a two-dimensional time-state vector as input.
The critic network outputs a scalar, and the value-function
approximator is defined by
\begin{align*}
    J^{\theta}(t, x) &= f_{\mathrm{critic}}^{\theta}(t, x) \boldsymbol{1}_{\{t<T\}} - K_3 x \boldsymbol{1}_{\{t=T\}}. 
\end{align*}
This construction ensures that
$J^\theta(T,x)=-K_3x=h(x)$ for every $x\in\mathcal{X}$.
The actor network outputs an $(S+1)$-dimensional vector,  with its $a$-th component denoted by
$f_{\mathrm{actor}}^\psi(t, x)[a]$. The policy and $q$-function
approximator are defined by
\begin{align*}
    \bpi^\psi(a\mid t,x)
    &=
    \frac{
        \exp\{f_{\mathrm{actor}}^\psi(t, x)[a+1] / \gamma
        \}
    }{\sum_{a'=0}^{S}
        \exp\{f_{\mathrm{actor}}^\psi(t, x)[a'+1] / \gamma
        \}
    },\qquad 
    q^\psi(t,x,a)=\gamma\log\bpi^\psi(a\mid t,x).
\end{align*}
This construction ensures that $q^\psi$ satisfies the normalization
condition in \eqref{eq:approx-condition}.

Besides $J^{\theta}$ and $q^{\psi}$, the remaining experimental parameters are configured as: initial state $x_0 = 0$, time step $\Delta t = 0.01$, learning rates $\alpha_{\theta} = 1\times 10^{-3}$, $\alpha_{\phi} = 2\times 10^{-6}$ and temperature parameter $\gamma = 1\times 10^{-2}$. 

During training, we periodically evaluate the current policy $\bpi$ by estimating $V(0,0;\bpi)$ from the average reward over $10,000$ independently simulated sample paths initialized at $x_0=0$. Figure \ref{fig:queueing_example} reports the estimated policy values over the learning process and shows the benchmark $V_{0.0001}^*(0,0)=-30.796$ as a dashed line. The policy values increase and stabilize near the benchmark, with the final learned policy achieving an average reward of $-31.272$. This corresponds to a small performance gap of $0.476$, indicating that the learned policy closely approximates the optimal policy.

\section{Proofs of Statements}\label{app:proofs-of-statements}
This section provides proofs of the main theoretical results. We first state a lemma on entropy-regularized maximization, which is needed for the proofs of Proposition~\ref{prop:HJB} and Theorem~\ref{thm:optimal-J-q}. We then proceed to prove Proposition~\ref{prop:HJB}, Theorem~\ref{thm:Dynkin}, and Theorem~\ref{thm:optimal-J-q} in turn.

\begin{lemma}\label{lem:optimal-policy-for-general}
    Let $\gamma > 0$, and a measurable function $\hat{q}: \mathbb{M} \mapsto \mathbb{R}$ with $\int_{\calA} \exp{\{\frac{1}{\gamma} \hat{q}(t, x, a)\} \mu(\mathrm{d}a)} < \infty$ for each $(t, x) \in [0, T] \times \calX$. Define the policy $\bpi^* \in \boldsymbol{\Pi}$ by
    \begin{align*}
        \bpi^*(a \mid t, x) \triangleq \frac{\exp\{\frac{1}{\gamma} \hat{q}(t, x, a)\}}{\int_{\calA} \exp\{\frac{1}{\gamma} \hat{q}(t, x, a')\} \mu(\mathrm{d}a')}, \quad \text{ for } (t, x, a) \in \mathbb{M}.
    \end{align*}
    Then, $\bpi^*$ is the unique maximizer of the following problem
    \begin{align*}
        \max_{\bpi \in \boldsymbol{\Pi}} \int_{\calA} [\hat{q}(t, x, a) - \gamma \log \bpi(a \mid t, x)] \bpi(a \mid t, x) \mu(\mathrm{d}a).
    \end{align*}
\end{lemma}
\begin{proof}
The result follows from the same argument as Lemma~13 in \citet{jia2023qlearning}, with Lebesgue measure replaced by
the reference measure $\mu$.
\end{proof}

\subsection{Proof of Proposition~\ref{prop:HJB}}

Let $\rho \triangleq 2L + b + c + 1$ with $L \triangleq \sup_{x \in \calX} \lambda^*(x)$ and $b$, $c$ as in Assumption \ref{asp:transtion-reward-rate}. We define an operator $\mathcal{G}$ on $B_{w}([0, T] \times \calX)$ as follows: for each $\psi \in B_{w}([0, T] \times \calX)$ and $(t, x) \in [0, T] \times \calX$,
\begin{align}\label{eq:contraction-operator-given-policy}
    \mathcal{G} \psi(t, x) \triangleq {}&e^{\rho t} \int_t^T \sup_{\bpi \in \boldsymbol{\Pi}} \int_{\calA} \{ R(s, x, a) + e^{-\rho s}\sum_{y \in \calX} \psi(s, y) \lambda(y \mid s, x, a) - \gamma \log \bpi(a \mid s, x) \} \bpi(\mathrm{d}a \mid s, x) \mathrm{d}s \nonumber \\
    {}&+e^{\rho t} h(x).
\end{align}
We need to verify the operator $\mathcal{G}$ is well-defined. Indeed, by Assumption \ref{asp:transtion-reward-rate} and Definition \ref{def:admissible-policy}(ii), we have
\begin{align}
    &\bigg\vert \int_{\calA} \{ R(s, x, a) + e^{-\rho s}\sum_{y \in \calX} \psi(s, y) \lambda(y \mid s, x, a) - \gamma \log \bpi(a \mid s, x) \} \bpi(\mathrm{d}a \mid s, x) \bigg\vert \mathrm{d}s \nonumber \\
    \leq {}& \bigg(\int_{\calA} \{\vert R(s, x, a) \vert + \sum_{y \in \calX} \vert \psi(s, y) \vert \vert \lambda(y \mid s, x, a)\vert \}  \bpi(\mathrm{d}a  \mid s, x) + \gamma \bigg\vert \int_{\calA} - \log\bpi(a \mid s, x) \bpi(\mathrm{d}a \mid s, x) \bigg\vert \bigg) \mathrm{d}s \nonumber \\
    \leq {}& [M_1 + \Vert \psi \Vert_{w} (c + b + 2L) + \gamma M_2 ] w(x), \quad \forall\, (t, x) \in [0, T] \times \calX, \label{eq:pf-integrand-bounded}
\end{align}
which implies that the integrand on the right-hand side of \eqref{eq:contraction-operator-given-policy} is Lebesgue integrable with respect to $t$, and thus $\mathcal{G}\psi(t, x)$ defined in \eqref{eq:contraction-operator-given-policy} is absolutely continuous and Borel measurable. 
Moreover, by combining \eqref{eq:contraction-operator-given-policy}, \eqref{eq:pf-integrand-bounded} and the fact that $\vert h(x) \vert \leq M_1 w(x)$ for all $x \in \calX$, we have
\begin{align*}
    \vert \mathcal{G}\psi(t, x) \vert \leq e^{\rho T} [M_1 + M_1 T + \Vert \psi \Vert_{w} T(c + b + 2L) + \gamma M_2 T] w(x), \quad \forall\, (t, x) \in [0, T] \times \calX,
\end{align*}
which implies that $\Vert \mathcal{G}\psi \Vert_{w} < \infty$. Hence, $\mathcal{G}\psi \in B_{w}([0, T] \times \calX)$, and $\mathcal{G}$ is a well-defined operator on $B_{w}([0, T] \times \calX)$. 

Next, we establish that $\mathcal{G}$ is a contraction operator. For any $\psi_1, \psi_2 \in B_{w}([0, T] \times \calX)$, it follows from \eqref{eq:contraction-operator-given-policy} and $\sum_{y \in \calX} \lambda(y \mid t, x, a) = 0$ that
\begin{align*}
    &\vert \mathcal{G} \psi_1(t, x) - \mathcal{G}\psi_2(t, x) \vert \\
    \leq {}& e^{\rho t} \int_t^T  e^{-\rho s} \cdot \sup_{\bpi \in \Pi} \int_{\calA} \bigg(\sum_{y \in \calX} \vert \psi_1(s, y) - \psi_2(t, y) \vert \vert \lambda(y \mid s, x, a)  \vert \bigg) \bpi(\mathrm{d}a \mid s, x) \mathrm{d}s \\
    \leq {}& e^{\rho t} \int_t^T e^{-\rho s} \Vert \psi_1 - \psi_2 \Vert_{w} \cdot \sup_{\bpi \in \Pi} \int_{\calA}\bigg( \sum_{y \in \calX} w(y) \lambda(y \mid s, x, a) + 2 w(x) \lambda(s, x, a)\bigg) \bpi(\mathrm{d}a \mid s, x) \mathrm{d}s \\
    \leq {}& e^{\rho t} \int_t^T e^{-\rho s} \Vert \psi_1 - \psi_2 \Vert_{w} [cw(x) + b + 2Lw(x)] \mathrm{d}s \\
    \leq {}& \frac{2L + b + c}{\rho} [1 - e^{-\beta (T-t)}] \Vert \psi_1 - \psi_2 \Vert_{w} w(x) \\
    \leq {}& \frac{2L + b + c}{\rho} \Vert \psi_1 - \psi_2 \Vert_{w} w(x).
\end{align*}
Hence, we obtain 
\begin{align*}
    \Vert \mathcal{G}\psi_1 - \mathcal{G}\psi_2 \Vert_{w} \leq \frac{2L + b + c}{2L + b + c + 1} \Vert \psi_1 - \psi_2 \Vert_{w}. 
\end{align*}
Since $\frac{2L + b + c}{2L + b + c + 1} < 1$,  $\mathcal{G}$ is a contraction operator on the Banach space $B_{w}([0, T] \times \calX)$.
Let $\psi^{*} \in B_{w}([0, T] \times \calX)$ denote the unique fixed point of $\mathcal{G}$.
That is, $\psi^*$ satisfies
\begin{align}\label{eq:contranction-fixed-point}
    \psi^*(t, x) ={}&
    e^{\rho t} \int_t^T \sup_{\bpi \in \Pi} \int_{\calA} \{ R(s, x, a) + e^{-\rho s}\sum_{y \in \calX} \psi^*(s, y) \lambda(y \mid s, x, a) - \gamma \log \bpi(a \mid s, x) \} \bpi(\mathrm{d}a \mid s, x) \mathrm{d}s \nonumber \\
    {}&+ e^{\rho t} h(x).
\end{align}
Define $\varphi(t, x) \triangleq e^{-\rho t} \psi^*(t, x)$ for all $(t, x) \in [0, T] \times \calX$. By construction of $\varphi$ and \eqref{eq:contranction-fixed-point}, we see that $\varphi \in B_{w}([0, T] \times \calX)$, and $\varphi(t, x)$ is differentiable in $t$ almost everywhere for each fixed $x \in \calX$ and satisfies the differential equation \eqref{eq:HJB}. Then, it follows from \eqref{eq:HJB} and \eqref{eq:pf-integrand-bounded} that $\frac{\partial \varphi}{\partial t}(t, x)$ is universally measurable and $w$-bounded on $[0, T] \times \calX$.
Thus $\varphi \in C^{1, 0}_{w}([0, T] \times \calX)$.
Hence, we have established the existence of a solution to \eqref{eq:HJB}. 

On the other hand, for any $\boldsymbol{\pi} \in \boldsymbol{\Pi}$ and any solution $\varphi \in C^{1, 0}_{w}([0, T] \times \calX)$ to \eqref{eq:HJB}, we first show that $J(t, x; \boldsymbol{\pi}) \leq \varphi(t,x)$. Since $\varphi$ satisfies Equation \eqref{eq:HJB}, we have
\begin{align}\label{eq:pf-HJB-le}
\frac{\partial \varphi}{\partial t}(t, x) + \int_{\calA} \{ H(t,x, a, \varphi(\cdot, \cdot) ) - \gamma \log \boldsymbol{\pi}(a \mid t, x)\} \boldsymbol{\pi}(\mathrm{d}a \mid t, x)  \leq  0. 
\end{align}
Then, by combining the definition of $J(t, x; \bpi)$ in \eqref{eq:value-function-J-given-pi}, inequality \eqref{eq:pf-HJB-le} and Theorem 3.1 in \cite{guo2015finite}, we derive that
\begin{align*}
    J(t, x; \bpi) &= \mathbb{E} \bigg[  \int_t^T  \int_{\calA}  [ R(s, \tilde X_{s-}^{\boldsymbol{\pi}}, a) - \gamma \log \boldsymbol{\pi}(a \mid s, \tilde X_{s-}^{\boldsymbol{\pi}})  ] \boldsymbol{\pi}(\mathrm{d}a \mid s, \tilde X_{s-}^{\boldsymbol{\pi}})  \mathrm{d}s +   h(\tilde X_T^{\boldsymbol{\pi}}) \mid \tilde X_t^{\boldsymbol{\pi}} = x \bigg]  \\
    &\leq \mathbb{E}\bigg[ \int_t^T  \bigg( - \frac{\partial \varphi}{\partial s }(s, \tilde{X}_{s-}^{\bpi}) - \sum_{y \in \calX} \varphi(s, y) \lambda^{\bpi} (y \mid s, \tilde{X}_{s-}^{\bpi})  \bigg) \mathrm{d}s + h(\tilde X_T^{\boldsymbol{\pi}}) \mid \tilde{X}_t^{\bpi} = x\bigg] \\
    &= \varphi(t, x),
\end{align*}
where $\lambda^{\bpi}(y \mid t, x)$ is as defined in \eqref{eq:exploratory-transition-rate}.
Thus, we have $J^*(t, x) \triangleq \sup_{\bpi \in \boldsymbol{\Pi}} J(t, x; \bpi) \leq \varphi(t, x)$ for all $(t, x) \in [0, T] \times \calX$.
Moreover, it follows from Lemma \ref{lem:optimal-policy-for-general} that the equality in \eqref{eq:pf-HJB-le} holds for policy $\bpi^*$ defined by
\begin{align*}
    \bpi^*(a \mid t, x) = \frac{\exp \{\frac{1}{\gamma} H(t, x, a, \varphi(\cdot, \cdot))\}}{\int_{\calA} \exp\{\frac{1}{\gamma} H(t, x, a', \varphi(\cdot, \cdot))\} \mu(\mathrm{d}a')}, \quad \text{ for } (t, x, a) \in \mathbb{M}.
\end{align*}
Hence we conclude that $J^*(t, x) = J(t, x; \bpi^*) = \varphi(t, x)$ for all $(t, x) \in [0, T] \times \calX$.
Then, the optimal policy $\bpi^*$ admits the form given in \eqref{eq:optimal-policy}.

\subsection{Proof of Theorem \ref{thm:Dynkin}}
Let
$\mathbb S=\{t=t_0<t_1<\cdots<t_K=T\}$
be a time grid on $[t,T]$, and let $A_k$ denote the action selected at
$t_k$ for $k=0,\ldots,K-1$.

For any $t_k\le r_1\le r_2\le t_{k+1}$, 
$A_k$ remains fixed on
$[u,v]$.
Therefore, conditional on $X_u^{\bpi,\mathbb S}=z$ and $A_k=a$, the process
$\{X_s^{\bpi,\mathbb S}:s \in [u, v]\}$ has the same law as
the
exploratory state process $\{\tilde X_s^{\delta_a}:s \in [u, v]\}$ under the deterministic
Markov policy $\delta_a$, started from $\tilde X_u^{\delta_a}=z$.
Indeed, for any $a \in \calA$, the
transition rates of $\{\tilde X_s^{\delta_a}:s \in [u, v]\}$ satisfy
\begin{align*}
\lambda^{\delta_{a}}(y\mid s,z)
=
\int_{\mathcal A}
\lambda(y\mid s,z,a')\,\delta_{a}(\mathrm d{a'}) =
\lambda(y\mid s,z,a),
\qquad s\in[u,v].
\end{align*} 
Applying Theorem~3.1 of \citet{guo2015finite} to the process $\{\tilde X_s^{\delta_a}:s \in [u, v]\}$ with
initial state $\tilde X_u^{\delta_{a}}=z$ yields
\begin{align*}
\mathbb E_{u,z}^{\delta_{a}}
\bigg[
    &\varphi(v,\tilde X_v^{\delta_{a}})-\varphi(u,z) -
    \int_u^v
    \mathcal L_s^{a}\varphi(s,\tilde X_{s-}^{\delta_{a}}) 
    \mathrm ds\bigg]  = 0,
\end{align*}
where
\(\mathcal L_s^a\varphi(s,x)
\triangleq
\frac{\partial\varphi}{\partial s}(s,x)
+
\sum_{y\in\calX}
\varphi(s,x)\lambda(y\mid s,x,a)\) for all $s\in[u, v]$, $x\in \calX$, and $a \in \calA$.
By the conditional equality in law stated above, we have, for $\{X_s^{\bpi,\mathbb S}:s \in [u, v]\}$ that,
\begin{align*}
\mathbb E
\bigg[
    &\varphi(v, X_v^{\bpi,\mathbb S})-\varphi(u,z) -
    \int_u^v
    \mathcal L_s^{a}\varphi(s, X_{s-}^{\bpi,\mathbb S}) 
    \mathrm ds \mid X_u^{\bpi,\mathbb S}=z, A_k=a\bigg] = 0.
\end{align*}
Moreover, 
the conditional dynamics of
$\{X_s^{\bpi,\mathbb S}:s \in [u, v]\}$ given $\mathcal F_u^{\mathbb S}$,
are completely determined by the current state
$X_u^{\bpi,\mathbb S}$ and the action $A_k$.
It follows that 
\begin{align}\label{eq:local-dynkin-grid}
\mathbb E
\bigg[
    &\varphi(v, X_v^{\bpi,\mathbb S})-\varphi(u,X_u^{\bpi,\mathbb S}) -
    \int_u^v
    \mathcal L_s^{A_k}\varphi(s, X_{s-}^{\bpi,\mathbb S}) 
    \mathrm ds \mid \calF_u^{\mathbb{S}}\bigg] = 0.
\end{align}
Define
\begin{align}\label{eq:dynkin-martingale}
M_s^\varphi
\triangleq
\varphi(s,X_s^{\bpi,\mathbb S})-\varphi(t,x) -
\int_t^s
\mathcal L_r^{a_r^{\bpi,\mathbb S}}
\varphi(r,X_{r-}^{\bpi,\mathbb S})\,\mathrm dr,
\qquad s\in[t,T].
\end{align}
Under Assumption~\ref{asp:transtion-reward-rate} and
$\varphi\in C_w^{1,0}([0,T]\times\mathcal X)$, the process
$M^\varphi$ is integrable. Since
$a_s^{\bpi,\mathbb S}=A_k$ on $(t_k,t_{k+1}]$,
\eqref{eq:local-dynkin-grid} implies that
\begin{align}\label{eq:local-martingale-increment}
\mathbb E[
M_v^\varphi-M_u^\varphi
\mid
\mathcal F_u^{\mathbb S}]
=0,
\end{align}
whenever $t_k\leq u\leq v\leq t_{k+1}$ for some
$k\in\{0,\ldots,K-1\}$.

Next, for any $t\leq u\leq v\leq T$, insert all grid points in $\mathbb S$ that lie strictly between $u$ and $v$ and denote the resulting partition of $[u,v]$ by
\[
u=\sigma_0<\sigma_1<\cdots<\sigma_N=v.
\]
Then each subinterval $[\sigma_n,\sigma_{n+1}]$ is contained in some grid interval
$[t_k,t_{k+1}]$.
By \eqref{eq:local-martingale-increment} and the tower
property,
\begin{align*}
\mathbb E[
    M_v^\varphi-M_u^\varphi
    \mid
    \mathcal F_u^{\mathbb S}] =
    \sum_{n=0}^{N-1}
    \mathbb E\left[
    \mathbb E[
    M_{\sigma_{n+1}}^\varphi-M_{\sigma_n}^\varphi
    \mid 
    \mathcal F_{\sigma_n}^{\mathbb S}
    ]
    \mid
    \mathcal F_u^{\mathbb S}
\right] =0.
\end{align*}
Hence, $\{M_s^\varphi:s\in[t,T]\}$ is an
$\{\mathcal F_s^{\mathbb S}\}_{s\in[t,T]}$-martingale.
In particular, for any $s\in[t,T]$,
\begin{align*}
0
&=
\mathbb E[
    M_T^\varphi-M_s^\varphi
    \mid
    \mathcal F_s^{\mathbb S}] \\
&=
\mathbb E\left[
    \varphi(T,X_T^{\bpi,\mathbb S})
    -\varphi(s,X_s^{\bpi,\mathbb S})
    -
    \int_s^T
    \mathcal L_u^{a_u^{\bpi,\mathbb S}}
    \varphi(u,X_{u-}^{\bpi,\mathbb S}) \mathrm du
    \mid
    \mathcal F_s^{\mathbb S}
\right].
\end{align*}
This yields the desired result.

\subsection{Proof of Theorem \ref{thm:optimal-J-q}}
To prove Theorem~\ref{thm:optimal-J-q}, we first establish the following auxiliary lemma. 
\begin{lemma}\label{lem:auxiliary-martingale}
Suppose Assumption \ref{asp:transtion-reward-rate} holds with function $w$. Then, for any $\bpi \in \boldsymbol{\Pi}$, initial time-state pair $(t, x) \in [0, T] \times \calX$, and time grid $\mathbb S=\{t=t_0<\cdots<t_K=T\}$ on $[t, T]$, the process  
\begin{align*}
M_s^{\rho} \triangleq \sum_{y\in\mathcal X}\sum_{z\neq y}
\int_{(t,s]}
\rho(u,y,a_u^{\bpi,\mathbb S},z)
\,\mathrm{d}N_{yz}^{\bpi,\mathbb S}(u) -
\int_t^s
\sum_{z\neq X_{u-}^{\bpi,\mathbb S}}
\rho(
    u,X_{u-}^{\bpi,\mathbb S},
    a_u^{\bpi,\mathbb S},z
)
\lambda(
    z\mid u,X_{u-}^{\bpi,\mathbb S},
    a_u^{\bpi,\mathbb S}
)
\mathrm du,
\quad s\in[t,T].
\end{align*}
is an
$\{\mathcal F_s^{\mathbb S}\}_{s\in[t,T]}$-martingale. 
\end{lemma}
\begin{proof}[Proof of Lemma~\ref{lem:auxiliary-martingale}]
Regarding $w$ as the time-independent function
$(s,x)\mapsto w(x)$, we have
$w\in C_w^{1,0}([0,T]\times\mathcal X)$. It follows from Theorem~\ref{thm:Dynkin} that 
\begin{align*}
    w(X_s^{\bpi, \mathbb{S}}) - \int_t^s \sum_{z \in \calX} w(z) \lambda(z \mid u, X_{u-}^{\bpi, \mathbb{S}}, a_{u}^{\bpi, \mathbb{S}})) \mathrm{d}u, \qquad s \in [t, T],
\end{align*}
is an $\{\calF_s^{\mathbb{S}}\}_{s \in [t, T]}$-martingale.
Hence, for every $s\in[t,T]$,
\begin{align*}
\mathbb E[w(X_s^{\bpi,\mathbb S}) \mid X_t^{\bpi,\mathbb S}=x]
={}&
w(x)
+
\mathbb E\bigg[
\int_t^s
\sum_{z\in\mathcal X}
w(z)\lambda(
    z\mid u,X_{u-}^{\bpi,\mathbb S},
    a_u^{\bpi,\mathbb S})
\,\mathrm du
\mid X_t^{\bpi,\mathbb S}=x \bigg].
\end{align*}
From Assumption~\ref{asp:transtion-reward-rate}(i), it follows that 
\begin{align*}
\mathbb E[w(X_s^{\bpi,\mathbb S}) \mid X_t^{\bpi,\mathbb S}=x]
\leq
w(x)
+
\int_t^s\bigl(
    c\,\mathbb E[w(X_{u-}^{\bpi,\mathbb S}) \mid X_t^{\bpi,\mathbb S}=x]+b\bigr)\mathrm du.
\end{align*}
Applying Gronwall's inequality yields
\begin{align}
\label{eq:grid-w-moment-bound}
\mathbb E[w(X_s^{\bpi,\mathbb S}) \mid X_t^{\bpi,\mathbb S}=x]
\leq e^{c(s-t)}w(x)
+
\frac{b}{c}[e^{c(s-t)}-1],
\qquad s\in[t,T].
\end{align}
Combing \eqref{eq:grid-w-moment-bound} with Assumption~\ref{asp:transtion-reward-rate}(iii), we obtain the following
integrability condition:
\begin{align}
\label{eq:pf-jump-reward-integrability}
{}&\mathbb E\bigg[
\int_t^T
\sum_{z\neq X_{s-}^{\bpi,\mathbb S}}
\vert
\rho(
    s,X_{s-}^{\bpi,\mathbb S},
    a_s^{\bpi,\mathbb S},z)
\vert \, \lambda(z\mid s,X_{s-}^{\bpi,\mathbb S}, a_s^{\bpi,\mathbb S})
\mathrm ds \,\Big|\, X_t^{\bpi,\mathbb S}=x
\bigg] \nonumber \\
\leq {}&M_1\int_t^T
\mathbb E[w(X_{s-}^{\bpi,\mathbb S}) \mid X_t^{\bpi,\mathbb S}=x]\mathrm ds < \infty.
\end{align}

For any $k\in\{0,\ldots,K-1\}$, $A_k$ is
$\mathcal F_{t_k}^{\mathbb S}$-measurable and remains fixed on
$(t_k,t_{k+1}]$. Therefore, given $t_k\leq u\leq v\leq t_{k+1}$, conditional on $\mathcal F_u^{\mathbb S}$, the process
$\{X_s^{\bpi,\mathbb S}:s \in [u, v]\}$ is a jump
Markov process starting from $X_u^{\bpi,\mathbb S}$ with transition
rates $\lambda(y\mid s,X_{s-}^{\bpi,\mathbb S},A_k)$.
By the integrability condition in \eqref{eq:pf-jump-reward-integrability} and the martingale property of stochastic integrals with respect to compensated counting processes
$\{
N_{yz}^{\bpi,\mathbb S}(s)
-
\int_t^s
\mathbf{1}_{\{X_{r-}^{\bpi,\mathbb S}=y\}}
\lambda(z\mid r,y,A_k)\,\mathrm{d}r: s \in [u, v]\}$ for $y,\, z\in\mathcal{X}$ with $y\neq z$, we obtain
\begin{align}\label{eq:pf-interval-martingale}
\mathbb E\bigg[
&\sum_{y\in\mathcal X}\sum_{z\neq y}
\int_{(u,v]}
\rho\bigl(s,y,A_k,z\bigr)
\,\mathrm{d}N_{yz}^{\bpi,\mathbb S}(s) \nonumber\\
&-
\int_u^v
\sum_{z\neq X_{s-}^{\bpi,\mathbb S}}
\rho(s,X_{s-}^{\bpi,\mathbb S},A_k,z)\,\lambda(z\mid s,X_{s-}^{\bpi,\mathbb S},A_k)
\,\mathrm ds \Bigm| \calF_u^{\mathbb S}\bigg]
=0.
\end{align}
Since
$a_s^{\bpi,\mathbb S}=A_k$ on $(t_k,t_{k+1}]$, \eqref{eq:pf-interval-martingale} implies that
\begin{align*}
\mathbb E[
M_v^\rho-M_u^\rho
\mid
\mathcal F_u^{\mathbb S}]
=0.
\end{align*}
Then, using the same argument as in the proof of Theorem~\ref{thm:Dynkin}, we conclude that $\{M_s^{\rho}: s\in [t, T]\}$ is an
$\{\mathcal F_s^{\mathbb S}\}_{s\in[t,T]}$-martingale. 
\end{proof}
With Lemma~\ref{lem:auxiliary-martingale} in hand, we now prove Theorem~\ref{thm:optimal-J-q}.
\begin{itemize}
    \item [(i)]
    Let $J^*$ and $q^*$ be the optimal value function and the optimal $q$-function, respectively.
    For any $\bpi \in \boldsymbol{\Pi}$, initial time-state pair $(t, x) \in [0, T] \times \calX$, and time grid $\mathbb{S}$ on $[t, T]$, we consider the process
    \begin{align*}
        M_s^* \triangleq J^*(s, X_s^{\bpi,\mathbb{S}}) + \int_t^s [ R(u, X_{u-}^{\bpi,\mathbb{S}}, a_u^{\bpi,\mathbb{S}}) - q^*(u, X_{u-}^{\bpi,\mathbb{S}}, a_u^{\bpi,\mathbb{S}}) ] \mathrm{d} u, \quad s \in [t, T].
    \end{align*} 
    where the $R(t, x, a)$ is the augmented running reward rate defined in~\eqref{eq:augmented-reward}.  
    For any $s \in [t, T]$, we have 
    \begin{align*}
        &\E [M_T^* - M_s^* \mid \calF_s^{\mathbb{S}}] \\
        = {}&\E \bigg[J^*(T, X_T^{\bpi,\mathbb{S}}) - J^*(s, X_s^{\bpi,\mathbb{S}}) - \int_s^T [ R(u, X_{u-}^{\bpi,\mathbb{S}}, a_u^{\bpi,\mathbb{S}}) - q^*(u, X_{u-}^{\bpi,\mathbb{S}}, a_u^{\bpi,\mathbb{S}}) ] \mathrm{d}u \mid \calF_s^{\mathbb{S}} \bigg] \\
        = {}&\E \bigg[J^*(T, X_T^{\bpi,\mathbb{S}}) - J^*(s, X_s^{\bpi,\mathbb{S}}) - \int_s^T \bigg(\frac{\partial J^*}{\partial u}(u, X_{u-}^{\bpi,\mathbb{S}}) + \sum_{y \in \calX} J^*(u, y) \lambda(y \mid u, X_{u-}^{\bpi,\mathbb{S}}, a_{u}^{\bpi,\mathbb{S}}) \bigg) \mathrm{d}u \mid \calF_s^{\mathbb{S}} \bigg] \\
        = {}&0,
    \end{align*}
    where the last equality follows directly from Theorem \ref{thm:Dynkin}.
    Hence, $\{M_s^*: s\in [t, T]\}$ is an $\{\calF_s^{\mathbb{S}}\}_{s \in[t, T]}$-martingale. Combining this with Lemma~\ref{lem:auxiliary-martingale}, we conclude that $\{M_s^*+M_s^{\rho}: s\in [t, T]\}$ is an $\{\calF_s^{\mathbb{S}}\}_{s \in[t, T]}$-martingale.

    \item [(ii)] 
    For any given $x \in \calX$, let $\{X_t^{\bpi,\mathbb S}: t \in [0, T]\}$ and $\{a_t^{\bpi,\mathbb S}: t \in [0, T]\}$ denote the grid sample state and action processes generated under $\bpi$ and $\mathbb S$, with $X_0^{\bpi,\mathbb S}=x$.
    Denote by $\mathbb P_x$ the corresponding probability law.  
    
    By the assumption in statement, the process \eqref{eq:optimal-martingale}
    is an $\{\calF_t^{\mathbb{S}}\}_{t \in [0, T]}$-martingale.
    Theorem \ref{thm:Dynkin} and Lemma~\ref{lem:auxiliary-martingale} imply that the following two processes are $\{\calF_t^{\mathbb{S}}\}_{t \in [0, T]}$-martingales:
    \begin{gather*}
        \widehat{J^*}(t, X_t^{\bpi, \mathbb{S}}) - \int_0^t \bigg(\frac{\partial \widehat{J^*}}{\partial s}(s, X_{s-}^{\bpi, \mathbb{S}}) + \sum_{y \in \calX} \widehat{J^*}(s, y) \lambda(y \mid s, X_{s-}^{\bpi, \mathbb{S}}, a_{s}^{\bpi, \mathbb{S}}) \bigg) \mathrm{d}s, \quad t \in [0, T], \\
        \sum_{y\in\mathcal X}\sum_{z\neq y}
        \int_{(0,t]}
        \rho(s,y,a_s^{\bpi,\mathbb S},z) \,\mathrm{d}N_{yz}^{\bpi,\mathbb S}(s) -
        \int_0^t
        \sum_{z\neq X_{s-}^{\bpi,\mathbb S}}
        \rho(
            s,X_{s-}^{\bpi,\mathbb S},
            a_s^{\bpi,\mathbb S},z
        )
        \lambda(
            z\mid s,X_{s-}^{\bpi,\mathbb S},
            a_s^{\bpi,\mathbb S}
        )
        \mathrm ds,
        \quad t\in[0,T].
    \end{gather*}
    Taking appropriate sum and difference of the three martingales above, we derive that the process $\{\int_0^t D(s, X_{s-}^{\bpi, \mathbb{S}}, a_{s}^{\bpi, \mathbb{S}}) \mathrm{d}s: t \in [0, T]\}$
    is a continuous finite-variation martingale, where
    \[
    D(t, x, a) \triangleq \frac{\partial \widehat{J^*}}{\partial t}(t,x) +\sum_{y\in\mathcal{X}} \widehat{J^*}(t,y)\lambda(y\mid t,x,a) +R(t,x,a)-\widehat{q^*}(t,x,a), \qquad (t, x, a) \in \mathbb{M}. 
    \]
    Therefore,
    \begin{align*}
        D(t, X_{t-}^{\bpi, \mathbb{S}}, a_t^{\bpi, \mathbb{S}})=0 \qquad \mathrm{d}t\otimes\mathbb{P}_x\text{-a.e.}.
    \end{align*}
    Tonelli's theorem gives
    \begin{align}\label{eq:int-D-equal-0}
    \int_{t_k}^{t_{k+1}}
    \mathbb E_x\bigl[
        \vert D(t, X_{t-}^{\bpi, \mathbb{S}}, a_t^{\bpi, \mathbb{S}}) \vert
    \bigr]\mathrm{d}t = 0, \qquad \forall\ x \in \calX.
    \end{align}
    Next, we show that, for every $x\in \calX$,
    \begin{align}\label{eq:D-equal-0}
        D(t, x, a)=0, \qquad \mathrm{d}t\otimes\mu(\mathrm{d}a)\text{-a.e. on } [0, T]\times\calA.
    \end{align}
    We prove this by contradiction. Suppose \eqref{eq:D-equal-0} does not hold for some $y\in\calX$. Then there exists $k\in\{0,\ldots,K-1\}$ such that the set
    \begin{align}\label{eq:set-S-x}
    S_k^y
    \triangleq
    \left\{
    (t,a)\in(t_k,t_{k+1}]\times\calA:
    \vert D(t,y,a)\vert > 0 
    \right\}
    \end{align}
    has positive $\mathrm{d}t\otimes\mu(\mathrm{d}a)$-measure.
    We consider the grid sample state process $X^{\bpi, \mathbb{S}}$ initialized at $X_0^{\bpi, \mathbb{S}}=y$, and define
    \[
    E_k^y
    \triangleq
    \{
    X_t^{\bpi,\mathbb S}=y
    \text{ for all }t\in[0,t_k]
    \}.
    \]
    Since $\lambda^*(y) = \sup_{t \in [0, T],\, a \in \mathcal{A}} \lambda(t, y, a) < \infty$, it follows that 
    \begin{align*}
        \mathbb P_y(E_k^y) \geq \exp\{- \lambda^*(y) t_k\} > 0.
    \end{align*}
    Next, let
    \(
        \tau_k^y
        \triangleq
        \inf\{t>t_k:X_t^{\bpi,\mathbb{S}}\neq y\}
    \), conditional on \(E_k^y\) and the action $A_k=a$ selected at $t_k$, 
    we have
    \[
        P_k^y(a)
        \triangleq
        \mathbb{P}_y (\tau_k^y>t_{k+1}
        \mid E_k^y,\, A_k=a )
        =
        \exp\left\{-\int_{t_k}^{t_{k+1}}\lambda(t,y,a)\,\mathrm{d}t\right\} \geq e^{-\lambda^*(y)(t_{k+1}-t_k)}>0.
    \] 
    Since conditional on $E_k^y$, $A_k$ is sampled from $\bpi(\cdot\mid t_k,y)$, and $a_t^{\bpi,\mathbb{S}} = A_k$ for $t \in (t_k,t_{k+1}]$, it follows that
    \begin{align*}
        {}&\int_{t_k}^{t_{k+1}}\mathbb E_{y}\bigl[
        \vert D(t, X_{t-}^{\bpi, \mathbb{S}}, a_t^{\bpi, \mathbb{S}}) \vert\bigr]\, \mathrm{d}t \\
        \geq {}&\Prob(E_k^y) \cdot \int_{t_k}^{t_{k+1}}\mathbb E_{y}\Bigl[
        \vert D(t, X_{t-}^{\bpi, \mathbb{S}}, A_k) \vert \, \big| \, E_k^y \Bigr]\, \mathrm{d}t \\
        = {}&\Prob(E_k^y) \cdot \int_{t_k}^{t_{k+1}} \int_{\calA} \mathbb E_{y}\Bigl[
        \vert D(t, X_{t-}^{\bpi, \mathbb{S}}, a) \vert \, \big| \, E_k^y, A_k=a\Bigr] \bpi(\mathrm{d}a \mid t_k, y)\, \mathrm{d}t \\
        \geq {}&\Prob(E_k^y) \cdot \int_{t_k}^{t_{k+1}} \int_{\calA} \mathbb E_y\Bigl[
        \vert D(t, X_{t-}^{\bpi, \mathbb{S}}, a) \vert \boldsymbol{1}_{\{\tau_k^y > t_{k+1}\}} \, \big| \, E_k^y, A_k=a\Bigr] \bpi(\mathrm{d}a \mid t_k, y)\, \mathrm{d}t \\
        = {}& \Prob(E_k^y) \cdot \int_{t_k}^{t_{k+1}} \int_{\calA} \vert D(t, y, a) \vert \, P_k^y(a)\, \bpi(a \mid t_k, y)\, \mu(\mathrm{d}a) \mathrm{d}t.
    \end{align*}
    By the definition of $S_k^y$ in \eqref{eq:set-S-x}, we have $S_k^y \subset (t_k, t_{k+1}]\times \calA$, and hence
    \begin{align*}
        \int_{t_k}^{t_{k+1}}\mathbb E_y\bigl[
        \vert D(t, X_{t-}^{\bpi, \mathbb{S}}, a_t^{\bpi, \mathbb{S}}) \vert\bigr]\, \mathrm{d}t \geq \Prob(E_k^y) \cdot \int_{S_k^y} \vert D(t, y, a) \vert \, P_k^y(a)\, \bpi(a \mid t_k, y)\, \mu(\mathrm{d}a) \mathrm{d}t. 
    \end{align*}
    Since $\bpi(\cdot \mid t_k, y) \sim \mu$ for every $\bpi \in \boldsymbol{\Pi}$, we have the density $\bpi(a \mid t_k, y) > 0$ for $\mu$-a.e. on $\calA$. 
    Moreover, $S_k^y$ has positive $(\mathrm{d}t \otimes \mathrm{d}\mu)$-measure, and $\vert D(t, y, a) \vert > 0$ on $S_k^y$, $P_k^y(a)>0$ for all $a\in \calA$, and $\Prob(E_k^y) > 0$, it follows that
    \begin{align*}
        \int_{t_k}^{t_{k+1}}\mathbb E_y\bigl[
        \vert D(t, X_{t-}^{\bpi, \mathbb{S}}, a_t^{\bpi, \mathbb{S}}) \vert\bigr]\, \mathrm{d}t > 0. 
    \end{align*}
    This contradicts the result derived in \eqref{eq:int-D-equal-0}, and hence the desired result \eqref{eq:D-equal-0} holds.

    From \eqref{eq:D-equal-0}, we have
    \begin{align}\label{eq:pf-optimal-J-q-equal-0}
        \widehat{q^*}(t,x,a) = \frac{\partial \widehat{J^*}}{\partial t}(t,x) + H(t,x,a, \widehat{J^*}(\cdot, \cdot)), \qquad \mathrm{d}t\otimes\mu(\mathrm{d}a)\text{-a.e. on } [0, T]\times\calA.
    \end{align}
    Since $\widehat{q^*}$ satisfies the normalization condition in \eqref{eq:optimal-J-q-constraints},
    we have
    \begin{align*}
    1
    = \int_{\mathcal A} \exp \{\frac{1}{\gamma}\widehat{q^*}(t,x,a)\}\mu(\mathrm{d}a) = \exp\{\frac{1}{\gamma}\frac{\partial \widehat J^*}{\partial t}(t,x)\} \cdot 
    \int_{\mathcal A} \exp\{\frac{1}{\gamma}H(t,x,a,\widehat J^*(\cdot, \cdot))\}\mu(\mathrm{d}a).
    \end{align*}
    Taking logarithms on both sides yields
    \begin{align}\label{eq:pf-optimal-J-q-exploratory-HJB}
        \frac{\partial \widehat{J^*}}{\partial t}(t, x) +  \gamma \log\bigg[ \int_{\calA}  \exp \{ \frac{1}{\gamma}H(t,x, a, \widehat{J^*}(\cdot, \cdot) ) \} \mu(\mathrm{d}a) \bigg]= 0.
    \end{align}
    Now, we define a policy $\widehat{\bpi^*} \in \boldsymbol{\Pi}$ as follows
    \begin{align*}
        \widehat{\bpi^*}(a \mid t, x) \triangleq \frac{\exp\{\frac{1}{\gamma} H(t, x, a, \widehat{J^*}(\cdot, \cdot))\}}{\int_{\calA} \exp\{\frac{1}{\gamma} H(t, x, a, \widehat{J^*}(\cdot, \cdot))\} \mu(\mathrm{d}a)}, \qquad (t, x, a) \in \mathbb{M}.
    \end{align*}
    It follows from Lemma \ref{lem:optimal-policy-for-general} that
    \begin{align}\label{eq:pf-optimal-J-q-HJB-part}
        &\sup_{ \boldsymbol{\pi} \in \boldsymbol{\Pi}} \int_{\calA} \{ H(t,x, a, \widehat{J^*}(\cdot, \cdot) ) - \gamma \log \boldsymbol{\pi}(a \mid t, x)\} \boldsymbol{\pi}(\mathrm{d}a \mid t, x) \nonumber \\ 
        = {}&\int_{\calA} \{ H(t,x, a, \widehat{J^*}(\cdot, \cdot) ) - \gamma \log \widehat{\bpi^*}(a \mid t, x)\} \widehat{\bpi^*}(\mathrm{d}a \mid t, x) \nonumber \\
        = {}&\gamma \log\bigg[ \int_{\calA}  \exp \{ \frac{1}{\gamma}H(t,x, a, \widehat{J^*}(\cdot, \cdot) ) \} \mu(\mathrm{d}a) \bigg] \nonumber \\
        = {}& - \frac{\partial \widehat{J^*}}{\partial t}(t, x),
    \end{align}    
    where the last equality is due to \eqref{eq:pf-optimal-J-q-exploratory-HJB}.
    Since $\widehat{J^*}$ satisfies Equation \eqref{eq:pf-optimal-J-q-HJB-part} and the terminal condition $\widehat{J^*}(T, x) = h(x)$ in \eqref{eq:optimal-J-q-constraints}, and in light of Proposition~\ref{prop:HJB}, we conclude that $\widehat{J^*}$ is the optimal value function. Moreover, from \eqref{eq:pf-optimal-J-q-equal-0} and the definition of the optimal $q$-function, we conclude that $\widehat{q^*}$ is the optimal $q$-function.
\end{itemize}
\end{document}